\documentclass[11pt]{article}

\usepackage[UKenglish]{babel}
\usepackage{comment}
\usepackage{amsmath}
\usepackage{amssymb}
\usepackage{amsthm}
\usepackage{proof}
\usepackage{graphicx}
\usepackage{color}
\usepackage{hyperref}

\hypersetup{
   colorlinks,%
   citecolor=blue,%
   filecolor=black,%
   linkcolor=red,%
   urlcolor=black
 }
	
\usepackage[all]{xy}
\usepackage{tikz} 
\usetikzlibrary{trees}
\usepgflibrary{arrows}
\usetikzlibrary{positioning,shapes}
\usetikzlibrary{patterns}

\usepackage{times}

\newcommand{\red}[1]{\textcolor{red}{#1}}

\renewcommand{\@}{\red{@}}

\newcommand{\EL}{\textbf{EL}}

\newcommand{\PLKw}{\textbf{PLKw}}

\newcommand{\BP}{\textbf{P}}

\newcommand{\Ag}{\textbf{I}}

\newcommand{\PLKwA}{\ensuremath{\textbf{PLKwA}}}

\newcommand{\SPS}{\ensuremath{\mathbb{S}}}

\newcommand{\SPLKw}{\ensuremath{\mathbb{PLKW}}}
\newcommand{\SPLKwT}{\ensuremath{\mathbb{PLKWT}}}
\newcommand{\SPLKwTr}{\ensuremath{\mathbb{PLKW}4}}

\newcommand{\SPLKwEuc}{\ensuremath{\mathbb{PLKW}5}}

\newcommand{\SPLKwTrEuc}{\ensuremath{\mathbb{PLKW}45}}
\newcommand{\SPLKwTTr}{\ensuremath{\mathbb{PLKWS}4}}
\newcommand{\SPLKwTEuc}{\ensuremath{\mathbb{PLKWS}5}}
\newcommand{\SPLKwA}{\ensuremath{\mathbb{PLKWA}}}

\newcommand{\SPLKwATEuc}{\ensuremath{\mathbb{PLKWAS}5}}

\newcommand{\lr}[1]{\langle #1 \rangle}

\newcommand{\toall}{\{\to_i\mid i\in \Ag\}}

\newcommand{\toallc}{\{\to^c_i\mid i\in \Ag\}}

\newcommand{\toallp}{\{\to'_i\mid i\in \Ag\}}

\newcommand{\mc}[1]{\mathcal{#1}}
\newcommand{\M}{\mc{M}}
\newcommand{\N}{\mc{N}}
\newcommand{\F}{\mc{F}}

\newcommand{\EquiKw}{\ensuremath{{\texttt{Kw}\!\lra}}}

\newcommand{\TAUT}{{\texttt{TAUT}}}

\newcommand{\CCOM}{{\texttt{!COM}}}

\newcommand{\GENKw}{\texttt{NECKw}}

\newcommand{\RE}{{\texttt{Sub}}}
\newcommand{\KwT}{\texttt{KwT}}
\newcommand{\KwTr}{\texttt{Kw4}}
\newcommand{\KwEuc}{\texttt{Kw5}}
\newcommand{\REKw}{\texttt{REKw}}
\newcommand{\ATOM}{{\texttt{!\!ATOM}}}
\newcommand{\NEG}{{\texttt{!\!NEG}}}

\newcommand{\AKw}{\texttt{!\!Kw}}
\newcommand{\AAA}{\texttt{!\!!}}

\newcommand{\MP}{{\texttt{MP}}}

\newcommand{\Tr}{\texttt{wKw4}}
\newcommand{\Euc}{\texttt{wKw5}}

\newcommand{\Kw}{\ensuremath{\textit{Kw}}}

\newcommand{\KwCon}{\texttt{KwCon}}
\newcommand{\KwDis}{\texttt{KwDis}}

\newcommand{\K}{\ensuremath{\textit{K}}}

\newcommand{\Lra}{\ensuremath{\Leftrightarrow}}

\newcommand{\lra}{\ensuremath{\leftrightarrow}}

\newtheorem{theorem}{Theorem}
\newtheorem{lemma}[theorem]{Lemma}
\newtheorem{definition}[theorem]{Definition}

\newtheorem{proposition}[theorem]{Proposition}

\newtheorem{corollary}[theorem]{Corollary}

\usepackage{a4wide}

\newcommand{\weg}[1]{}
\newcommand{\PLKwK}{\textbf{PLKwK}}
\renewcommand{\phi}{\varphi}

\newcommand{\SIg}{\mathbf{Ig}}

\title{Knowing Whether}
\author{Jie Fan, Yanjing Wang\thanks{Corresponding author}, Hans van Ditmarsch}
\date{}

\begin{document}
\maketitle

\begin{abstract}
Knowing whether a proposition is true means knowing that it is true or knowing that it is false. In this paper, we study logics with a modal operator $\Kw$ for knowing whether but without a modal operator $\K$ for knowing that. This logic is not a normal modal logic, because we do not have $\Kw(\phi\to \psi)\to(\Kw\phi\to\Kw\psi)$. Knowing whether logic cannot define many common frame properties, and its expressive power is less than that of basic modal logic over classes of models without reflexivity. These features make axiomatizing knowing whether logics non-trivial. We axiomatize knowing whether logic over various frame classes. We also present an extension of knowing whether logic with public announcement operators and we give corresponding reduction axioms for that. We compare our work in detail to two recent similar proposals.
\end{abstract}


\section{Introduction}

The work entitled `Logics of public communication' by Plaza \cite{plaza:1989}\footnote{Reprinted as \cite{plaza:2007} with comments by \cite{hvd.plaza:2007}.} is mainly known as one of the founding publications, if not {\em the} founding publication, of public announcement logic. However, it also treats two other topics worthy of investigation, namely `knowing value' modal operators and their binary variant `knowing whether' modal operators. You know the pincode of your bankcard if you know the \emph{value} of it. You know \textit{whether} $p$ if you know that $p$ is true or you know that $p$ is false. Plaza demonstrates the use of these operators in his discussion of Sum-and-Product puzzle, and poses as an open question what the axiomatization would be of public announcement logic with these operators \cite[p.13]{plaza:1989}. In \cite{wangetal:2013}, the authors investigate knowing value operators in depth and give a complete axiomatization of the logic with knowing value and public announcement, where the knowing value modality behaves quite differently from a modality in a normal modal logic, such as the standard knowledge modality. Unlike the knowing value modality, knowing whether is definable in terms of knowing that. But it still seems interesting to investigate a logic with a knowing whether operator but without the usual knowledge operator, and this motivates this paper.

Knowing whether operators have been discussed in other settings in the logical literature. \cite{Hart:1996} uses `knowing whether' operator to establish a continuum of knowledge states in a very neat fashion, which demonstrates that `knowing whether' is more convenient to use than `knowing that' in certain contexts, as argued also in \cite{heifetz1993universal}.  In natural language `knowing whether' is frequently used instead of `knowing that'. You say: ``I know whether it is raining outside,'' but you do not say: ``I know that it is raining outside or I know that it is not raining outside.'' We often only need to know whether someone knows the truth value of something, rather than the truth value itself. For example, a conference organizer needs to make sure that he knows whether the invited participants will come: it is fine for him as long as a definitive confirmation is given. An analysis in inquisitive semantics of such natural expressions is found in  \cite{alonietal:2013}. For another example, consider the Muddy Children Puzzle \cite{mosesetal:1986}. By iterating the announcement of the formula ``nobody knows whether he or she is muddy,'' this formula will finally become false, i.e., the muddy children will finally learn that they are muddy. This can be succinctly said in a logic with knowing whether and public announcements as primitives. Finally, consider gossip protocols \cite{hedetniemietal:1988}, wherein processors (or `agents') exchange information by one-to-one communications (`calls') wherein they exchange the value of their local state. We can assume the information to be exchanged are propositional secrets which have binary values. If $p$ describes the secret of agent 1, then after agent 2 calls agent 1, both 1 and 2 know whether $p$. Whether the value is true or false is irrelevant for the design of such protocols.

The logic of knowing whether has also appeared in a different form in prior literature, namely as the {\em logic of ignorance} \cite{wiebeetal:2003,hoeketal:2004,steinsvold:2008}: you are \emph{ignorant} about a proposition iff you do \emph{not know whether} the proposition is true. An axiomatization of the logic of ignorance over the class of arbitrary frames is given in \cite{wiebeetal:2003,hoeketal:2004}. The authors suggest that it is hard to repeat this exercise for other frame classes. In this work, we advance the study of knowing whether logic by systematically axiomatizing knowing whether logic over various common frame classes. We discuss in detail the difference between \cite{wiebeetal:2003,hoeketal:2004,steinsvold:2008} and our results.

Knowing whether logic is not a normal modal logic, it cannot define common frame properties, and it is less expressive than the basic modal logic, although equally expressive on reflexive models. We give a complete axiomatization, and also various extensions for special frame classes, and an addition of the logic with public announcements.

\medskip

In Section \ref{sec.logic} we define the language and semantics of the logic of knowing whether. Section \ref{sec.expr} deals with expressivity over models and frames, and Section \ref{sec.axiomatization} presents our new axiomatization of the logic over the class of arbitrary frames and proves its completeness. Then, in Section \ref{sec.extensions} we give axiomizations for other frame classes---highly non-trivial in this setting, as frame properties are not definable, unlike in standard modal logic. In Section \ref{sec.announcement} we extend knowing whether logic with public announcements, and in Section \ref{sec.comparison} we discuss the literature on the logic of ignorance in relation to our results.

\section{Syntax and semantics of the logic of knowing whether} \label{sec.logic}

We define the logical language in a more general setting including knowing whether but also knowledge. However, we will focus on the language with only knowing whether.
\begin{definition}[Logical languages  \PLKwK, \PLKw\ and \EL] \label{def.language} Let a set $\BP$ of propositional variables and a set $\Ag$ of agents be given. The logical language $\PLKwK(\BP,\Ag)$ is defined as:
$$\phi::=\top\mid p\mid\neg\phi\mid(\phi\wedge\phi)\mid\Kw_i\phi\mid\K_i\phi$$
where $p\in\BP$ and $i\in\Ag$. Without the $\K_i\phi$ construct, we have the {\em language $\PLKw(\BP,\Ag)$ of knowing whether logic}. Without the $\Kw_i\phi$ construct, we have the {\em language $\EL(\BP,\Ag)$ of epistemic logic}.
\end{definition}
We typically omit the parameters $\BP$ and $\Ag$ from the notations for these languages. The formula $\K_i\phi$ stands for `agent $i$ knows that $\phi$,' although we do not restrict ourselves to an epistemic context. The formula $\Kw_i\phi$ stands for `agent $i$ knows whether $\phi$'. As usual, we define $\bot$, $(\phi\vee\psi)$, $(\phi\to\psi)$, $(\phi\lra\psi)$ as the abbreviations of, respectively, $\neg\top$, $\neg(\neg\phi\land\neg\psi)$, $(\neg\phi\vee\psi)$, and $((\phi\to\psi)\land(\psi\to\phi))$. We omit parentheses from formulas unless confusion results. In particular, we assume that $\land$ and $\vee$ bind stronger than $\to$ and $\lra$. For $\phi_1\land\cdots\land\phi_m$ we write $\bigwedge_{j=1}^m\phi_j$, and for $\phi_1\vee\cdots\vee\phi_m$ we write $\bigvee_{j=1}^m\phi_j$.

\begin{definition}[Model]
A \emph{model} is a triple $\M=\langle S,\toall, V\rangle$ where $S$ is a non-empty set of possible worlds, $\to_i$ is a binary relation over $S$ for each $i \in \Ag$, and $V$ is a valuation function assigning a set of worlds $V(p)\subseteq S$ to  each $p\in \BP$. Given a world $s \in S$, a pair $(\M,s)$ is a \emph{pointed model}. A \emph{frame} is a pair $\F=\langle S,\toall \rangle$, i.e., a model without a valuation. We will refer to special classes of models or frames using the notation below. A binary relation is \emph{partial-functional} iff it corresponds to a partial function, i.e., every world has at most one successor.
$$
\begin{array}{|l|l|}
  \hline
  \text{\rm Notation }& \text{\rm Frame Property}\\
  \hline
  \mathcal{K} & \text{---} \\
  \hline
  \mathcal{D} & \text{seriality} \\
  \hline
  \mathcal{T} & \text{reflexivity} \\
  \hline
  \mathcal{B}& \text{symmetry}\\
  \hline
  4 & \text{transitivity} \\
  \hline
  5 & \text{euclidicity}\\
  \hline
  45& \text{transitivity, euclidicity}\\
  \hline
  \mathcal{S}4&\text{reflexivity, transitivity}\\
  \hline
  \mathcal{S}5&\text{reflexivity, euclidicity}\\
\hline
  \mathcal{PF} &\text{partial functionality}\\
  \hline
\end{array}
$$
\end{definition}
We will omit parenthesis around pointed models $(\M,s)$ whenever convenient. The non-standard notion of partial functionality plays a special role in knowing whether logics.

\begin{definition}[Semantics]
Given a model $\M=\langle S,\toall, V\rangle$, the semantics of \PLKwK\  is defined as follows:
\begin{center}
$
\begin{array}{|lcl|}
\hline
\mc{M},s\vDash \top  & \Leftrightarrow &   \textrm{ always }\\
\mc{M},s\vDash p  & \Leftrightarrow &s \in V(p) \\
\mc{M},s\vDash \neg\phi &\Leftrightarrow& \mc{M},s\nvDash \phi \\
\mc{M},s\vDash \phi\land \psi &\Leftrightarrow&\mc{M},s\vDash \phi \textrm{ and } \mc{M},s\vDash \psi \\
\mc{M},s\vDash\Kw_i\phi&\Leftrightarrow&\text{for all }t_1,t_2\text{ such that }s\to_it_1,s\to_it_2: \\ &&(\mc{M},t_1\vDash\phi\Leftrightarrow\mc{M},t_2\vDash\phi)\\
\M,s\vDash\K_i\phi&\Lra&\text{for all }t \text{ such that }s\to_it: \M,t\vDash\phi \\
\hline
\end{array}
$
\end{center}
If $\M,s\vDash\phi$ we say that $\phi$ is true in $(\M,s)$, and sometimes write $s\vDash\phi$ if $\M$ is clear; if for all $s$ in $\M$ we have $\M,s\vDash\phi$ we say that $\phi$ is \emph{valid on $\M$} and write $\M\vDash\phi$; if for all $\mathcal{M}$ based on $\mathcal{F}$ with $\M\vDash\phi$ we say that $\phi$ is \emph{valid on $\mathcal{F}$} and write $\mathcal{F}\vDash\phi$; if for all $\mathcal{F}$ with $\F\vDash\phi$, $\phi$ is \emph{valid} and we write $\vDash\phi$. Given $\Phi \subseteq \PLKwK$, $\M,s\models\Phi$ stands for `for all $\phi\in\Phi$, $\M,s\models \phi$,' and similarly for model/frame validity, and validity. If there exists an $(\M,s)$ such that $\M,s\vDash\phi$, then $\phi$ is \emph{satisfiable}.
\end{definition}
Intuitively, $\Kw_i\phi$ is true at $s$ if and only if $\phi$ has the same truth value on the worlds that $i$ thinks possible. Knowing whether logic is not normal, because $\Kw_i(\phi\to\psi)\to(\Kw_i\phi\to\Kw_i\psi)$ is invalid (and, in relation to that, $\vDash\phi\to\psi$ does not imply $\vDash\Kw_i\phi\to\Kw_i\psi$). In the countermodel $\M_1$ below we have that $\M_1,s\vDash\Kw_i(p\to q)$ and $\M_1,s\vDash\Kw_ip$, but $\M_1,s\nvDash\Kw_iq$.

\medskip

$$
\xymatrix{\M_1: \ \ \ \  {s:\neg p,\neg q}\ar@(ul,ur) \ar[rr]   &  &          {\neg p, q}  }
$$

We use $\phi[\psi/p]$ to denote a {\em uniform substitution} of $\phi$, i.e., the formula obtained by replacing all occurrences of $p$ in $\phi$ (if there is any) with $\psi$. It can be shown that uniform substitution preserves the validity of $\PLKw$-formulas.

\begin{proposition}\label{prop.us}
For any $\psi,\phi\in\PLKw$, any $p\in\BP$: if $\vDash\phi$, then $\vDash\phi[\psi/p]$.
\end{proposition}
\begin{proof}
We show the contrapositive, namely that $\nvDash\phi[\psi/p]$ implies $\nvDash\phi$.
Assume a pointed model $(\M,s)$ such that $\M,s\nvDash\phi[\psi/p]$, where $\M=\lr{S,\toall,V}$. Consider model $\M'$ that is as $\M$ but with valuation $V'$, where $V'(p)= \{ t \in S \mid \M,t \models \psi \}$. We show that for all $\chi$ and for all $s$ in the domain of $\M$: $\M,s\vDash\chi[\psi/p]$ iff $\M',s\vDash\chi$, by induction on $\chi$. The only non-trivial case is $\Kw_i\chi$.
$$
\begin{array}{lll}
\M,s\vDash\Kw_i\chi[\psi/p]&\text{iff} \\
\text{for all }s_1,s_2 \text{ such that }s\to_is_1,s\to_is_2: \M,s_1\vDash\chi[\psi/p]\text{ iff }\M,s_2\vDash\chi[\psi/p] & \text{iff} & \text{(by induction)} \\
\text{for all }s_1,s_2 \text{ such that }s\to_is_1,s\to_is_2: \M',s_1\vDash\chi\text{ iff }\M',s_2\vDash\chi & \text{iff} \\
\M',s\vDash\Kw_i\chi
\end{array}
$$
Therefore, from $\M,s\nvDash\phi[\psi/p]$ follows $\M',s\nvDash\phi$, and therefore $\nvDash\phi$, as desired.
\end{proof}

\section{Expressivity and frame correspondence} \label{sec.expr}

In this section we compare the relative expressivity of knowing whether logic and epistemic logic, and we give some negative results for frame correspondence for knowing whether logic.

\subsection{Expressivity}
We adopt the definition of expressivity in \cite[Def.8.2]{hvdetal.del:2007}.
\begin{definition}[Expressive] Given two logical languages $L_1$ and $L_2$ that are interpreted in the same class of models,
\begin{itemize}
\item $L_2$ is {\em at least as expressive as }$L_1$, notation $L_1\preceq L_2$, if and only if for every formula $\phi_1\in L_1$ there is a formula $\phi_2\in L_2$ such that $\phi_1\Leftrightarrow\phi_2$ (i.e., $\phi_1$ and $\phi_2$ are logically equivalent).
\item $L_1$ and $L_2$ are {\em equally expressive}, notation $L_1\equiv L_2$, if and only if $L_1\preceq L_2$ and $L_2\preceq L_1$.
\item $L_1$ is {\em less expressive than }$L_2$, notation $L_1\prec L_2$, if and only if $L_1\preceq L_2$ and $L_2\not\preceq L_1$.
\end{itemize}
\end{definition}

\begin{proposition}\label{prop.lessexpressiveone}
\PLKw\ is less expressive than \EL\ on the class of $\mathcal{K}$ models, $\mathcal{D}$-models, $4$-models, $5$-models.
\end{proposition}
\begin{proof}
This is a truth-preserving translation $t$ from \PLKw\ to \EL:
$$
\begin{array}{lll}
t(p)&=&p\\
t(\neg\phi)&=&\neg t(\phi)\\
t(\phi\wedge\psi)&=&t(\phi)\wedge t(\psi)\\
t(\Kw_i\phi)&=&\K_it(\phi)\vee\K_i\neg t(\phi)
\end{array}
$$
Therefore \EL\ is at least as expressive as \PLKw. But \PLKw\ is not at least as expressive as \EL: even the simplest \EL\ formula $\K_ip$ does not have an equivalent \PLKw\ correspondent. The pointed models $(\M,s)$ and $(\N,t)$ below, which are distinguished by $\K_ip$, cannot be distinguished by a \PLKw\ formula.

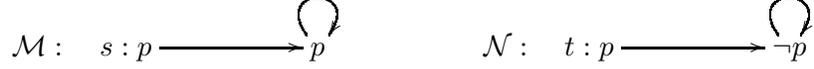
\medskip

$$
\xymatrix{ \M: \ \ \ \ {s:p} \ar[rr]   &  &          {p}\ar@(ul,ur)   && \N: \ \ \ \ {t:p} \ar[rr]&& {\neg p}\ar@(ul,ur) }
$$

\medskip

Note that $\mc{M}$ and $\mc{N}$ are serial, transitive, and Euclidean. By induction we prove that $\M,s$ and $\N,t$ are modally equivalent in $\PLKw$. The non-trivial case is  $\phi=\Kw_i\psi$. Note that $s$ and $t$ can only see one point. Therefore, $\M,s\vDash\Kw_i\psi$ and $\mc{N},t\vDash\Kw_i\psi$, so also, as required, $\M,s\vDash\Kw_i\psi$ iff $\mc{N},t\vDash\Kw_i\psi$.
\end{proof}

\begin{proposition}\label{prop.lessexpressivetwo}
\PLKw\ is less expressive than \EL\ on the class of $\mathcal{B}$-models.
\end{proposition}

\begin{proof}
Consider the following $\mathcal{B}$-models $(\M',s')$ and $(\N',t')$. Again, they are distinguished by $\K_ip$, but are modally equivalent in $\PLKw$ (by a similar argument as in Prop.~\ref{prop.lessexpressiveone}).

\medskip

$$
\xymatrix{ \M': \ \ \ \ {s':p} \ar[rr]   &  &          {p}\ar[ll]   && \N': \ \ \ \ {t':p} \ar[rr]&& {\neg p}\ar[ll]   }
$$
\end{proof}

However, on the class of $\mathcal{T}$-models, \PLKw\ and \EL\ are equally expressive.
\begin{proposition}\label{equallyexpressive}
\PLKw\ and \EL\ are equally expressive on the class of $\mathcal{T}$-models.
\end{proposition}

\begin{proof}
Consider translation $t':\EL\to\PLKw$:
$$
\begin{array}{lll}
t'(p)&=&p\\
t'(\neg\phi)&=&\neg t'(\phi)\\
t'(\phi\wedge\psi)&=&t'(\phi)\wedge t'(\psi)\\
t'(\K_i\phi)&=&t'(\phi)\wedge\Kw_it'(\phi)
\end{array}
$$
This translation $t'$ is truth preserving (elementary, by induction on $\phi$ in $t'(\phi)$). This demonstrates that $\EL \preceq \PLKw$. As we already had $\PLKw \preceq \EL$, by way of translation $t$ defined in the proof of Proposition~\ref{prop.lessexpressiveone}, we get that $\EL \equiv \PLKw$ on $\mathcal{T}$.
\end{proof}
This result applies to any model class contained in $\mathcal{T}$, such as $\mathcal{S}4$ and $\mathcal{S}5$.

\medskip

We close this section on expressivity with a curious observation related to (although not strictly about) expressivity. We now know that knowledge \emph{cannot} be defined in terms of knowing whether on ${\mathcal K}$, but that knowledge \emph{can} be defined in terms of knowing whether on ${\mathcal T}$. It is therefore interesting to observe that under slightly stronger conditions, knowledge can still be `defined' (in a different technical sense) in terms of knowing whether on ${\mathcal K}$, namely, given a model, in a world of that model wherein the agent is ignorant about something. Ignorant means `not knowing whether', so this implies that knowledge is definable in a world from where there are at least two accessible worlds.

\begin{proposition}\label{prop.iffK}
Assume that $\M,s\vDash \neg \Kw_i\psi$ for some $\psi$. Then: $\M,s\vDash \K_i \neg\phi$ if and only if there exists a $\chi$ such that $\M,s \vDash \Kw_i\phi\land \Kw_i(\phi\to\chi)\land \neg \Kw_i\chi$.
\end{proposition}

\begin{proof}
Suppose that $\M,s\vDash \neg \Kw_i\psi$ for some $\psi$. We need to show the equivalence.

First, assume there exists $\chi$: $\M,s \vDash \Kw_i\phi\land \Kw_i(\phi\to\chi)\land \neg \Kw_i\chi$. Suppose towards contradiction that $\M,s\nvDash\K_i\neg\phi$, then there exists $t$ such that $s\to_it$ and $t\vDash\phi$. Moreover, since $\M,s\vDash\neg \Kw_i\chi$, it follows for some $t_1,t_2$ with $s\to_it_1,s\to_it_2$ and $t_1\vDash\chi,t_2\vDash\neg\chi$. By the fact that $s\vDash\Kw_i\phi$, $s\to_it,s\to_it_1$ and $t\vDash\phi$, we get $t_1\vDash\phi$, similarly we can get $t_2\vDash\phi$, and thus $t_1\vDash\phi\to\chi$ but $t_2\nvDash\phi\to\chi$, contradicting the assumption that $\M,s\vDash\Kw_i(\phi\to\chi)$, as desired.

For the converse, assume $\M,s\vDash \K_i \neg\phi$. Then for all $t$ such that $s\to_it:\M,t\vDash\neg\phi$, thus $\M,t\vDash\phi\to\psi$. Therefore  $\M,s\vDash\Kw_i\phi$ and $\M,s\vDash\Kw_i(\phi\to\psi)$. It is clear $\M,s\vDash \neg \Kw_i\psi$ from the supposition. Then we can conclude that there exists $\chi$: $\M,s \vDash \Kw_i\phi\land \Kw_i(\phi\to\chi)\land \neg \Kw_i\chi$.
\end{proof}
Intuitively, we can `define' knowledge (the $\K_i\neg\phi$ in the proposition) in a given world $s$, iff there is some \PLKw\ formula $\psi$ that agent $i$ is ignorant about in $s$ (iff $\neg\Kw_i\psi$ is true in $s$), in other words, iff for any proposition whatsoever ($\psi$) there are two accessible worlds from $s$ with different values for it.

The property formulated in Prop.~\ref{prop.iffK} is important. It motivates the canonical model construction for knowing whether logic, as we will see in Section \ref{sec.axiomatization}.

\subsection{Frame correspondence}

Standard modal logic formulas can be used to capture frame properties, e.g., $\K p \to p$ corresponds to the reflexivity of frames. It is therefore remarkable that in knowing whether logic there is no such correspondence for most of the basic frame properties.  The authors of \cite{wiebeetal:2003} already demonstated that reflexivity is undefinable in the language of ignorance (which is equally expressive as $\PLKw$, see Section \ref{sec.comparison}). In this section we extend their result to other frame properties.

\begin{definition}[Frame definability] Let $\Phi$ be a set of $\PLKw$-formulas and $\mathrm{F}$ a class of frames. We say that $\Phi$ defines $\mathrm{F}$ if for all frames $\mathcal{F}$, $\mathcal{F}$ is in $\mathrm{F}$ if and only if $\mathcal{F}\vDash\Phi$. In this case we also say $\Phi$ defines the property of $\mathrm{F}$. If $\Phi$ is a singleton (e.g. $\phi$), we usually write $\mathcal{F}\vDash\phi$ rather than $\mathcal{F}\vDash\{\phi\}$. A class of frames (or the corresponding frame property) is definable in \PLKw\ if there is a set of $\PLKw$-formulas that defines it. \end{definition}

\begin{proposition}\label{prop.alsohandy}
Let $\F$ be a partial-functional frame and $\phi \in \PLKw$. Then $\F \models \Kw_i \phi$.
\end{proposition}
\begin{proof}
Given $\F=\langle S,\toall \rangle$, let $V$ be a valuation on $\F$ and $s \in S$. Because $s$ has at most one successor, the semantics of knowing whether gives us that $\F,V,s \models \Kw_i \phi$. (See also the countermodels used in the proof of Prop.~\ref{prop.lessexpressiveone}.)
\end{proof}
Consequently, we can view $\Kw_i\phi$ formulas as $\top$ on partial-functional frames. Therefore the only $\PLKw$ validities on partial-functional frames are essentially instantiations of tautologies which are the same on partial-functional models. A moment of reflection should confirm: 
\begin{corollary}\label{prop.yetalsohandy}
For any partial-functional frames $\F,\F'$ and any $\phi \in \PLKw$: $\F \models \phi$ iff $\F' \models \phi$.
\end{corollary}

\begin{proposition}\label{prop.undefinability}
The frame properties of seriality, reflexivity, transitivity, symmetry, and Euclidicity are not definable in \PLKw.
\end{proposition}

\begin{proof}
Consider the following frames:
\medskip
$$
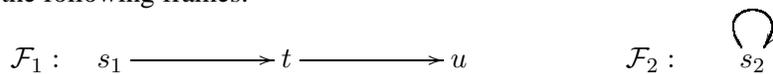
\xymatrix{{\mathcal{F}_1}: \ \ \ \ {s_1} \ar[rr]   &   &  {t}\ar[rr] &&  {u}  &&  \mathcal{F}_2: \hspace{-.2cm}  &  {s_2} \ar@(ul,ur)}
$$

\medskip

All two frames are partial-functional. So we have that, for any $\Phi\subseteq\PLKw$: ${\mathcal{F}_1} \models \Phi$ iff ${\mathcal{F}_2} \models \Phi$. Now observe that:
\begin{itemize}
\item ${\mathcal{F}_2}$ is reflexive but ${\mathcal{F}_1}$ is not reflexive.
\item ${\mathcal{F}_2}$ is serial but ${\mathcal{F}_1}$ is not serial.
\item ${\mathcal{F}_2}$ is transitive but ${\mathcal{F}_1}$ is not transitive.
\item ${\mathcal{F}_2}$ is symmetric but ${\mathcal{F}_1}$ is not symmetric.
\item ${\mathcal{F}_2}$ is Euclidean but ${\mathcal{F}_1}$ is not Euclidean.
\end{itemize}
The argument now goes as follows. Consider the first item, reflexivity: If $\Phi$ were to define reflexivity, then, as $\mathcal{F}_2$ is reflexive, we have ${\mathcal{F}_2} \models \Phi$. But as ${\mathcal{F}_2}$ and ${\mathcal{F}_1}$ satisfy the same frame validities, we also have that ${\mathcal{F}_1} \models \Phi$. However, ${\mathcal{F}_1}$ is not reflexive. Therefore such a $\Phi$ does not exist. Therefore, reflexivity is not frame definable in knowing whether logic.

The argument is similar for the other cases, using the other items in the list above. (Observe that ${\mathcal{F}_1}$ is indeed not Euclidean, because $s_1 \to t$ and $s_1 \to t$, but it is not the case that $t \to t$.)
\end{proof}

As a consequence of this result, the axiomatizations of knowing whether logics over special frame classes, such as the class of reflexive frames, cannot be shown by the standard method of adding the corresponding frame axioms to the axiomatization of $\PLKw$. This will be addressed in Section \ref{sec.extensions}.

\section{Axiomatization}\label{sec.axiomatization}

In this section we give a complete Hilbert-style proof system for the logic \PLKw\, on the class of all frames.

\subsection{Proof system and soundness}
\begin{definition}[Proof system \SPLKw] \label{axiomstable}
The proof system \SPLKw\ consists of the following axiom schemas and  inference rules.
\[ \begin{array}{ll}
\TAUT & \text{all instances of tautologies}  \\

\KwCon & \Kw_i(\chi\to\phi)\land\Kw_i(\neg\chi\to\phi)\to\Kw_i\phi \\

\KwDis & \Kw_i\phi\to \Kw_i (\phi\to \psi )\lor \Kw_i(\neg \phi\to \chi)\\

\EquiKw & \Kw_i\phi\lra\Kw_i\neg\phi \\

\MP & \text{From } \phi \text{ and } \phi\to\psi \text{ infer } \psi \\

\GENKw & \text{From } \phi \text{ infer } \Kw_i\phi \\

\REKw & \text{From } \phi\lra\psi \text{ infer } \Kw_i\phi\lra\Kw_i\psi
\end{array} \]

A {\em derivation} is a finite sequence of $\PLKw$ formulas such that each formula is either the instantiation of an axiom or the result of applying a inference rule to prior formulas in the sequence. A formula $\phi\in\PLKw$ is called {\em derivable}, or a {\em theorem}, notation $\vdash \phi$, if it occurs in a derivation.
\end{definition}

Intuitively, $\KwCon$ means if an agent knows whether a formula is implied not only by some formula but by its negation, then the agent also knows whether the formula holds; $\KwDis$ means if an agent knows whether a formula holds, then either the agent knows this formula holds, in which case the agent knows whether its negation implies any formula, or the agent knows it does not hold, in which case the agent knows whether it implies any formula; $\EquiKw$ means knowing whether a formula holds is same as knowing whether the formula does not hold.

\begin{proposition}\label{prop.sound} The proof system $\SPLKw$ is sound with respect to the class of all frames.\end{proposition}
\begin{proof}
The soundness of $\SPLKw$ follows immediately from the validity of three crucial axioms. The other axioms and the derivation rules are obviously valid. We prove that:
\begin{enumerate}
\item $\KwCon$ is valid: $\vDash\Kw_i(\chi\to\phi)\land\Kw_i(\neg\chi\to\phi)\to\Kw_i\phi$
\item $\KwDis$ is valid: $\vDash\Kw_i\phi\to \Kw_i (\phi\to \psi )\lor \Kw_i(\neg \phi\to \chi)$
\item $\EquiKw$ is valid: $\vDash\Kw_i\phi\lra\Kw_i\neg\phi$
\end{enumerate}
Proof:
\begin{enumerate}
\item
Assume towards a contradiction that for some $(\M,s)$ such that $\M,s\vDash\Kw_i(\chi\to\phi)$, $\M,s\vDash\Kw_i(\neg\chi\to\phi)$ but $\M,s\vDash\neg\Kw_i\phi$, then there exist $t_1,t_2$ such that $s\to_it_1,s\to_it_2$ and $t_1\vDash\phi,t_2\vDash\neg\phi$. Clearly, with $t_1\vDash\phi$ we get $t_1\vDash\chi\to\phi$ and $t_1\vDash\neg\chi\to\phi$. Thus from the fact that $s\vDash\Kw_i(\chi\to\phi),s\to_it_1,s\to_it_2$ and $t_1\vDash\chi\to\phi$ we get $t_2\vDash\chi\to\phi$. Similarly, by using $t_1\vDash\neg\chi\to\phi$ we can get $t_2\vDash\neg\chi\to\phi$. Now we obtain $t_2\vDash\chi\to\phi$ and $t_2\vDash\neg\chi\to\phi$, therefore $t_2\vDash\phi$. Contradiction. 

\item Let $(\M,s)$ be an arbitrary model. Suppose via contraposition that $\M,s\vDash\neg\Kw_i(\phi\to \psi )$ and $\M,s\vDash\neg\Kw_i(\neg \phi\to \chi)$, we only need to show  $\M,s\vDash\neg\Kw_i\phi$. By supposition, there exist $t_1,t_2$ such that $s\to_it_1,s\to_it_2$ and $t_1\vDash\phi\to \psi, t_2\vDash\neg(\phi\to\psi)$ and, there exist $u_1,u_2$ such that $s\to_iu_1,s\to_iu_2$ and $u_1\vDash\neg \phi\to \chi,u_2\vDash\neg(\neg\phi\to\chi)$, respectively. From $t_2\vDash\neg(\phi\to\psi)$ and $u_2\vDash\neg(\neg\phi\to\chi)$ it follows $t_2\vDash\phi$ and $u_2\vDash\neg\phi$ respectively. So far we have shown $s\to_it_2,s\to_iu_2$ and $t_2\vDash\phi,u_2\vDash\neg\phi$, therefore we conclude that $\M,s\vDash\neg\Kw_i\phi$, as desired.

\item This is immediate from the semantics of $\Kw_i$.
\end{enumerate}
\end{proof}

\begin{proposition}\label{replacementofequivalent}
Consider the inference rule {\em Substitution of equivalents}:
\[ \RE \ \ \ \text{From } \phi\lra\psi, \text{ infer } \chi[\phi/p]\lra\chi[\psi/p] \]
Substitution of equivalents is admissible in $\SPLKw$.
\end{proposition}

\begin{proof}
By induction on the structure of $\chi$. The non-trivial case is $\Kw_i\chi$. Suppose $\vdash\phi\lra\psi$. By inductive hypothesis we have $\vdash\chi[\phi/p]\lra\chi[\psi/p]$. Then, using $\REKw$, we get $\vdash\Kw_i(\chi[\phi/p])\lra\Kw_i(\chi[\psi/p])$, i.e. $\vdash \Kw_i\chi[\phi/p]\lra \Kw_i\chi[\psi/p]$.
\end{proof}

The inference rule $\REKw$ in the system $\SPLKw$ is crucial. Consider again the schema \[ \texttt{K} \ \ \ \Kw_i(\phi\to\psi)\to(\Kw_i\phi\to\Kw_i\psi) \] We have already shown in Section \ref{sec.logic} that $\texttt{K}$ is invalid. This axiom is typically used to prove $\RE$, but is lacking in \SPLKw. Without $\REKw$ (see the proof above) $\RE$ is not admissable in \SPLKw.

\medskip

We will now first derive a \SPLKw\ theorem (Proposition \ref{Mixi2}) that will play an important part in the completeness proof. To structure the derivation we employ two lemmas deriving \SPLKw\ theorems.

\begin{lemma}\label{Mix}
$\vdash\left( \ \Kw_i\chi \land \Kw_i(\neg\chi\to\phi) \land\neg\Kw_i\phi \land\Kw_i(\chi\to\psi) \ \right)\to \Kw_i\psi$
\end{lemma}

\begin{proof}
$$
\begin{array}{lll}
(i)&\Kw_i(\neg\chi\to\phi)\land\neg\Kw_i\phi\to\neg\Kw_i(\chi\to\phi)&\KwCon,
\TAUT\\
(ii)&\Kw_i\chi\to\Kw_i(\chi\to\phi)\vee\Kw_i(\neg\chi\to\psi)&\KwDis\\
(iii)&\Kw_i(\chi\to\psi)\land\Kw_i(\neg\chi\to\psi)\to\Kw_i\psi&\KwCon\\
(iv)&(\Kw_i\chi \land \Kw_i(\neg\chi\to\phi) \land\neg\Kw_i\phi)\to \Kw_i(\neg\chi\to\psi) &\TAUT,(i),(ii)\\
(v)&(\Kw_i\chi \land \Kw_i(\neg\chi\to\phi) \land\neg\Kw_i\phi \land\Kw_i(\chi\to\psi))\to \Kw_i\psi&\TAUT,(iii),(iv)
\end{array}
$$
\end{proof}

\begin{lemma}\label{Mixi}
\[ \vdash( \ \Kw_i(\phi\land\neg\psi\to\chi)\land \Kw_i\psi\land \Kw_i(\psi\to\delta)\land \neg\Kw_i\delta \ )\to \Kw_i(\phi\to\chi) \]
\end{lemma}

\begin{proof}
$$
\begin{array}{lll}
(i)&(\Kw_i\psi\land\Kw_i(\neg\psi\to(\phi\to\chi))\land\neg\Kw_i(\phi\to\chi)\land\Kw_i(\psi\to\delta))\to\Kw_i\delta& \text{Lemma}\ \ref{Mix}\\
(ii)&(\Kw_i(\neg\psi\to(\phi\to\chi))\land\Kw_i\psi\land\Kw_i(\psi\to\delta)\land
\neg\Kw_i\delta)\to\Kw_i(\phi\to\chi)&\TAUT,(i)\\
(iii)&(\phi\land\neg\psi\to\chi)\lra(\neg\psi\to(\phi\to\chi))&\TAUT\\
(iv)&\Kw_i(\phi\land\neg\psi\to\chi)\lra\Kw_i(\neg\psi\to(\phi\to\chi))&\REKw,(iii)\\
(v)&(\Kw_i(\phi\land\neg\psi\to\chi)\land \Kw_i\psi\land \Kw_i(\psi\to\delta)\land \neg\Kw_i\delta)\to \Kw_i(\phi\to\chi)&\RE,(ii),(iv)
\end{array}
$$
\end{proof}

\begin{proposition} \label{Mixi2}
For all $k\geq 1$:
\[\vdash\left(\bigwedge^k_{j=1}\Kw_i \chi_j\land \Kw_i ( \bigwedge^k_{j=1}\neg \chi_j\to \phi)\land \neg \Kw_i\phi \land \bigwedge^k_{j=1}\Kw_i(\chi_j\to \psi_j)\right)\to
\bigvee^k_{j=1}\Kw_i\psi_j \]
\end{proposition}

\begin{proof}
The proof is by induction on $k$. The base step $k=1$ is clear from Lemma \ref{Mix}. For the inductive step, assume by inductive hypothesis (IH) that the proposition holds for $k=m$. We now show that:
$$\vdash\left(\bigwedge^{m+1}_{j=1}\Kw_i \chi_j\land \Kw_i (\bigwedge^{m+1}_{j=1}\neg \chi_j\to \phi)\land \neg \Kw_i\phi \land \bigwedge^{m+1}_{j=1}\Kw_i(\chi_j\to \psi_j)\right)\to
\bigvee^{m+1}_{j=1}\Kw_i\psi_j$$
The proof is as follows.
$$
\begin{array}{lll}
1&(\bigwedge_{j=1}^m\Kw_i \chi_j\land \Kw_i (\bigwedge_{j=1}^m\neg \chi_j\to \phi)\land \neg \Kw_i\phi \\
&\ \ \ \ \land \bigwedge_{j=1}^m\Kw_i(\chi_j\to \psi_j))\to
\bigvee_{j=1}^m\Kw_i\psi_j&\text{IH}\\
2&(\Kw_i(\bigwedge_{j=1}^m\neg\chi_j\land\neg\chi_{m+1}\to\phi)\land \Kw_i\chi_{m+1}\land\Kw_i(\chi_{m+1}\to\psi_{m+1})\\
&\ \ \ \ \land \neg\Kw_i\psi_{m+1})\to \Kw_i(\bigwedge_{j=1}^m\neg\chi_j\to\phi)&\text{Lemma }\ref{Mixi} \\
3&(\Kw_i(\bigwedge_{j=1}^{m+1}\neg\chi_j\to\phi)\land \Kw_i\chi_{m+1}\land \Kw_i(\chi_{m+1}\to\psi_{m+1})\\
&\ \ \ \ \land \neg\Kw_i\psi_{m+1})\to \Kw_i(\bigwedge_{j=1}^m\neg\chi_j\to\phi)&\REKw, 2\\
4& (\Kw_i (\bigwedge_{j=1}^{m+1}\neg \chi_j\to \phi)\land \bigwedge_{j=1}^{m+1}\Kw_i \chi_j\land\bigwedge_{j=1}^{m+1}\Kw_i(\chi_j\to \psi_j)\\
&\ \ \ \ \land \neg\Kw_i\psi_{m+1}\land\neg \Kw_i\phi)\to(\bigwedge_{j=1}^m\Kw_i \chi_j\land \Kw_i (\bigwedge_{j=1}^m\neg \chi_j\to \phi)\\
&\ \ \ \ \land \neg \Kw_i\phi \land \bigwedge_{j=1}^m\Kw_i(\chi_j\to \psi_j))&3\\
5& (\Kw_i (\bigwedge_{j=1}^{m+1}\neg \chi_j\to \phi)\land \bigwedge_{j=1}^{m+1}\Kw_i \chi_j\\
&\ \ \ \ \land \bigwedge_{j=1}^{m+1}\Kw_i(\chi_j\to \psi_j)\land \neg\Kw_i\psi_{m+1}\land\neg \Kw_i\phi)\to \bigvee_{j=1}^{m}\Kw_i\psi_j&1,4\\
6&(\bigwedge_{j=1}^{m+1}\Kw_i \chi_j\land \Kw_i (\bigwedge_{j=1}^{m+1}\neg \chi_j\to \phi)\land \neg \Kw_i\phi \\
&\ \ \ \ \land \bigwedge_{j=1}^{m+1}\Kw_i(\chi_j\to \psi_j))\to
\bigvee_{j=1}^{m+1}\Kw_i\psi_j&5\\
\end{array}
$$
\end{proof}

\subsection{Completeness}

We proceed with the completeness of the proof system \SPLKw. The completeness of the logic is shown via a canonical model construction.
\begin{definition}[Canonical model]\label{cononicalmodel}
The canonical model $\M^c$ of \SPLKw~ is the tuple $\lr{S^c,\toallc,V^c}$, where:
\begin{itemize}
\item $S^c=\{s\mid s\text{ is a maximal consistent set of }\SPLKw\}$.
\item $s\to^c_it\text{ iff }$
\begin{enumerate}
\item there exists $\chi$ such that $\neg \Kw_i\chi\in s$  and
\item for all $\phi$ and $\psi$: $\Kw_i\phi\wedge\Kw_i(\phi\to\psi)\land \neg \Kw_i\psi\in s \text{ implies }\neg \phi\in t$.
\end{enumerate}
\item $V^c(p)=\{s\in S^c\mid p\in s\}$.
\end{itemize}
\end{definition}
We observe that every consistent set of $\SPLKw$ can be extended to a maximal consistent set of $\SPLKw$ (Lindenbaum Lemma) in the standard way. The binary relations between worlds in the canonical model are special. The definition is inspired by the canonical relation where $s\to^c_it$ iff for all $\phi$: $\K_i\phi\in s$ implies $\phi\in t$, and the observation of Proposition \ref{prop.iffK}, the `almost definability' of knowledge. We also use the contrapositive this condition: \begin{quote} {\em For every $i\in\Ag$, $s\to^c_it$ iff (1.) there exists $\chi$ such that $\neg \Kw_i\chi\in s$ and (2.) for all $\phi$ and $\psi$: if $\phi\in t$ then at least one of $\Kw_i\phi$, $\Kw_i(\phi\to\psi)$ and $\neg \Kw_i\psi$ is not in $s$.} \end{quote}

\begin{lemma}[Truth Lemma]\label{truthlem}
For any \PLKw~ formula $\phi$, $\M^c,s\vDash\phi$ iff $\phi\in s$.
\end{lemma}

\begin{proof}
By induction on $\phi$. The only non-trivial case is $\Kw_i\phi$.

``$\Leftarrow$'': Assume towards contradiction that $\Kw_i\phi\in s$ but $\M,s\vDash\neg \Kw_i\phi$, namely there are $t_1$ and $t_2$ such that $s\to^c_i t_1$ and $s\to^c_i t_2$ and $\M^c,t_1\vDash\phi$ and $\M^c,t_2\vDash\neg \phi$. From $\M^c,t_1\vDash\phi$ and $\M^c,t_2\vDash\neg \phi$, and the induction hypothesis, we infer that  $\phi\in t_1$ and $\neg \phi\in t_2$, respectively. From $s\to^c_i t_1$ and (1.)  we infer that there is a $\chi_1$ such that $\neg \Kw_i\chi_1\in s$. From that, the assumption $\Kw_i\phi\in s$ and (2.) follows that $\Kw_i(\phi\to\chi_1)\not\in s$, i.e., $\neg \Kw_i(\phi\to\chi_1)\in s$. Similarly, from $s\to^c_i t_2$ we derive that there is a $\chi_2$ such that $\neg \Kw_i(\neg\phi\to\chi_2)\in s$. From $\neg \Kw_i(\phi\to\chi_1),\neg \Kw_i(\neg\phi\to\chi_2) \in s$ and Axiom \KwDis\ we now have $\neg \Kw_i\phi\in s$. Contradiction.

``$\Rightarrow$'': Assume that $\Kw_i\phi\notin s$. To show that $s\nvDash\Kw_i\phi$, we need to construct two points $t_1,t_2\in S^c$ such that $s\to^c_i t_1, s\to^c_i t_2$ and $\phi\in t_1,\neg\phi\in t_2$.
First we have to show:
\begin{enumerate}
\item \label{oneone}$\{\neg \chi\mid\Kw_i\chi\wedge\Kw_i(\chi\to\psi)\land \neg \Kw_i \psi\in s \text{ for some }\psi\}\cup\{\neg\phi\}$  is consistent.
\item \label{twotwo}$\{\neg \chi\mid\Kw_i\chi\wedge\Kw_i(\chi\to\psi')\land \neg \Kw_i \psi'\in s \text{ for some }\psi'\}\cup\{\phi\}$  is consistent.
\end{enumerate}

We prove item \ref{oneone}. Suppose the set is inconsistent. Then there exist $\chi_1,\cdots,\chi_n$ and $\psi_1, \cdots, \psi_n$ such that $\vdash\neg \chi_1\wedge\cdots\land \neg \chi_n\to\phi$ and $\Kw_i\chi_k\wedge\Kw_i(\chi_k\to\psi_k)\land \neg \Kw_i \psi_k\in s$ for all $k\in [1,n]$. From \GENKw, we have $\Kw_i((\bigwedge_{k=1}^n\neg \chi_k)\to\phi)\in s$. Now since $\Kw_i\phi\notin s$, from the maximal consistency of $s$ we get $\neg\Kw_i\phi \in s$. Then, from Proposition \ref{Mixi2}\ we infer that $\bigvee_{k=1}^n \Kw_i\psi_k\in s$. And this contradicts that $\neg \Kw_i \psi_k\in s$ for all $k\in[1, n]$.

From item \ref{oneone}, the definition of the canonical relation, and the observation that every consistent set has a maximal consistent extension (Lindenbaum Lemma), we conclude that there exists a $t_2\in S^c$ such that $s\to^c_i t_2$ and $\neg\phi\in t_2$.

The proof of item \ref{twotwo} is similar to item \ref{oneone}, but we need to use  \EquiKw, and similarly, from item \ref{twotwo} we derive that there exists a $t_1\in S^c$ such that $s\to^c_i t_1$ and $\phi\in t_1$.
\end{proof}

\begin{theorem}[Completeness]
\SPLKw\ is complete with respect to the class ${\mathcal K}$ of all frames. That is, for every $\phi\in\PLKw$, $\vDash\phi$ implies $\vdash\phi$.
\end{theorem}
\begin{proof}
Suppose $\nvdash\phi$, then $\neg\phi$ is \SPLKw-consistent. By Lindenbaum-Lemma there exists $s\in S^c$ such that $\neg\phi\in s$, and thus $\M^c,s\vDash\neg\phi$ by Truth Lemma, therefore $\nvDash\phi$.
\end{proof}

Given the translation from $\PLKw$ to $\EL$ (i.e.\ the translation $t$ in the proof of Proposition \ref{prop.lessexpressiveone}), and the decidability of $\EL$, the (satisfiability problem of) knowing whether logic is obviously decidable.
\begin{proposition}[Decidability of $\PLKw$]\label{deciableknowingwhether}
The logic $\PLKw$ is decidable.
\end{proposition}

\section{Axiomatization: extensions} \label{sec.extensions}

In this section we will give extensions of $\SPLKw$ w.r.t.\ various classes of frames, and prove their completeness. Definition \ref{def.ext} shows the extra axiom schemas and corresponding systems, with on the right-hand side in the table the frame classes for which we will demonstrate completeness.

\begin{definition}[Extensions of \SPLKw]\label{def.ext}
\[ \begin{array}{|lll|l|}
  \hline
  \text{Notation}& \text{Axiom Schemas}& \text{Systems} & \text{Frames} \\
\hline
  \KwT & \Kw_i\phi\land\Kw_i(\phi\to\psi)\land\phi\to\Kw_i\psi& \SPLKwT=\SPLKw+\KwT & {\mathcal T} \\
  \KwTr & \Kw_i\phi\to\Kw_i(\Kw_i\phi\vee\psi)& \SPLKwTr=\SPLKw+\KwTr & 4\\
  \KwEuc & \neg\Kw_i\phi\to\Kw_i(\neg\Kw_i\phi\vee\psi)&\SPLKwEuc=\SPLKw+\KwEuc & 5\\
   \Tr  & \Kw_i\phi\to\Kw_i\Kw_i\phi &\SPLKwTTr=\SPLKw+\KwT+\Tr & {\mathcal S}4\\
   \Euc   & \neg\Kw_i\phi\to\Kw_i\neg\Kw_i\phi &\SPLKwTEuc=\SPLKw+\KwT+\Euc & {\mathcal S}5 \\
  \hline
\end{array}
\]
\end{definition}

\begin{proposition}\label{validthree} \
\begin{itemize}
\item $\KwT$ is valid on the class of all $\mathcal{T}$-frames;
\item $\KwTr$ is valid on the class of all $4$-frames;
\item $\KwEuc$ is valid on the class of all $5$-frames.
\end{itemize}
\end{proposition}

\begin{proof} \
\begin{itemize}
\item   Given any $\M=\lr{S,\toall,V}$ based on a reflexive frame and an $s\in S$, suppose $\M,s\vDash\Kw_i\phi\land\Kw_i(\phi\to\psi)\land\phi$.
Towards a contradiction assume $\M,s\nvDash\Kw_i\psi$, then there exist $t,t'$ such that $s\to_it,s\to_it'$ and $t\vDash\psi,t'\vDash\neg\psi$. From the reflexivity of $s$ it follows that $s\to_is$, and thus $t\vDash\phi, t'\vDash\phi$ by the facts that $s\vDash\Kw_i\phi\land\phi$, $s\to_it,s\to_it'$. Then it is easy to get $t\vDash\phi\to\psi$ but $t'\nvDash\phi\to\psi$, which contradicts the supposition $s\vDash\Kw_i(\phi\to\psi)$.

\item Given any $\M=\lr{S,\toall,V}$ based on a transitive frame and an $s\in S$, suppose that $\M,s\vDash\Kw_i\phi$. Towards a contradiction assume $\M,s\nvDash\Kw_i(\Kw_i\phi\vee\psi)$ for some $\psi$, then there exist $t,t'$ such that $s\to_it,s\to_it'$ and $t\vDash\Kw_i\phi\vee\psi,t'\vDash\neg\Kw_i\phi\land\neg\psi$. From $t'\vDash\neg\Kw_i\phi$ it follows that for some $u,u'$ such that $t'\to_iu,t'\to_iu'$ and $u\vDash\phi,u'\vDash\neg\phi$. From transitivity it follows that $s\to_iu,s\to_iu'$ due to the facts that $s\to_it',t'\to_iu$ and $t'\to_iu'$. Therefore $s\nvDash\Kw_i\phi$, contradicting the assumption.

\item Given any $\M=\lr{S,\toall,V}$ based on an Euclidean frame and an $s\in S$, suppose that $\M,s\vDash\neg\Kw_i\phi$. Towards a contradiction assume $\M,s\nvDash\Kw_i(\neg\Kw_i\phi\vee\psi)$. Then there exist $t,t'$ such that $s\to_it,s\to_it'$ and $t\vDash\neg\Kw_i\phi\vee\psi,t'\vDash\Kw_i\phi\land\neg\psi$. Moreover, it follows that for some $u,u'$ such that $s\to_iu,s\to_iu'$ and $u\vDash\phi,u'\vDash\neg\phi$ from the supposition. Since $\M$ is Euclidean, $t'\to_iu$ and $t'\to_iu'$ based on the facts that $s\to_it',s\to_iu,s\to_iu'$. Therefore $t'\nvDash\Kw_i\phi$ because $u\vDash\phi,u'\vDash\neg\phi$, contradicting $t'\vDash\Kw_i\phi$.
\end{itemize}
\end{proof}

To prove the soundness of the novel proof systems we only need to refer to the soundness of the axioms $\KwT$, $\KwTr$, and $\KwEuc$, as demonstrated in Proposition \ref{validthree}. Before we proceed to demonstrate the completeness of these systems, let us first give some intuitions and motivation to explain the form of the axioms. The reader might have expected a more familiar connection between frame classes and axioms instead:
\begin{itemize}
\item For reflexive frames: why not $\Kw_i\phi\to\phi$ instead of $\Kw_i\phi\land\Kw_i(\phi\to\psi)\land\phi\to\Kw_i\psi$?
\item For transitive frames: why not $\Kw_i\phi\to\Kw_i\Kw_i\phi$ instead of $\Kw_i\phi\to\Kw_i(\Kw_i\phi\vee\psi)$?
\item For Euclidean frames: why not $\neg\Kw_i\phi\to\Kw_i\neg\Kw_i\phi$ instead of $\neg\Kw_i\phi\to\Kw_i(\neg\Kw_i\phi\vee\psi)$?
\end{itemize}
First, we recall the reader that none of these frame classes are definable by knowing whether formulas (Prop.\ \ref{prop.undefinability}). So, the axioms in the proof systems defined above fulfil a different role, there is no correspondence in the standard modal logical sense. Second, the three `familiar' formula schemas tentatively stipulated above could of course still be valid. But are they? Sometimes yes, at other times no. Concerning $\Kw_i\phi\to\phi$: this is of course not valid in knowing whether logic (you may know whether $p$ because you know that $p$ is false). Further, trying to obtain an interesting validity by translating axiom $\K_i\phi\to\phi$ from epistemic logic into $\PLKw$ (according to the translation defined in Prop.\ \ref{equallyexpressive}) does not lead anywhere: we get $\phi\land \Kw_i\phi\to\phi$, a tautology. Concerning $\Kw_i\phi\to\Kw_i\Kw_i\phi$: this is valid on transitive frames. It is also derivable in the proof system \SPLKwTr\ defined above: it is essentially an instantiation of axiom $\KwTr$ for $\psi = \bot$.\footnote{Note that $\Kw_i\phi \vee \bot$ is equivalent to $\Kw_i\phi$ and we can apply the substitution of equivalents principle $\RE$ (this is admissable as well in all defined extensions of $\SPLKw$).} But just this principle $\Tr$ on top of \SPLKw\ was not enough to obtain completeness (Prop.\ \ref{incompleteness4} below), we do need the stronger version $\KwTr$. However, in the presence of $\KwT$, $\Tr$ is sufficient to demonstrate completeness for \SPLKwTTr\ (Prop.\ \ref{compextensionS4}), since we can actually derive $\KwTr$ in the system based on $\Tr$ and  $\KwT$ (Prop.\ \ref{prop.kw4kw5}). A similar story goes for $\Euc$ and $\KwEuc$: $\SPLKw$+$\Euc$ is not complete w.r.t.\ Euclidean frames but $\SPLKwTEuc=\SPLKw+\KwT+\Euc$ is complete w.r.t.\ $S5$-frames (Prop.\ \ref{incompleteness5} and Prop.\ \ref{compextensionS5}). 


The following proposition says that $\KwTr$ and $\KwEuc$ are derivable in $\SPLKwTTr$ and $\SPLKwTEuc$ respectively, which are crucial in the proofs of  Theorem~\ref{compextensionS4} and Theorem ~\ref{compextensionS5}, respectively. 

\begin{proposition}\label{prop.kw4kw5}
\begin{enumerate}
\item\label{prop.kw4} $\vdash_{\SPLKwTTr}\Kw_i\phi\to\Kw_i(\Kw_i\phi\vee\psi)$
\item\label{prop.kw5} $\vdash_{\SPLKwTEuc}\neg\Kw_i\phi\to\Kw_i(\neg\Kw_i\phi\vee\psi)$
\end{enumerate}
\end{proposition}

\begin{proof}
\begin{enumerate}
\item The following is a derivation in $\SPLKwTTr$:
\[ \begin{array}{lll}
(i)& \Kw_i\phi\to(\Kw_i\phi\vee\psi)& \TAUT\\
(ii)& \Kw_i(\Kw_i\phi\to(\Kw_i\phi\vee\psi))& \GENKw,(i)\\
(iii)& \Kw_i\Kw_i\phi\land \Kw_i(\Kw_i\phi\to(\Kw_i\phi\vee\psi))\land \Kw_i\phi\to \Kw_i(\Kw_i\phi\vee\psi)& \KwT\\
(iv)& \Kw_i\Kw_i\phi\land \Kw_i\phi\to \Kw_i(\Kw_i\phi\vee\psi)& (ii),(iii)\\
(v)& \Kw_i\phi\to\Kw_i\Kw_i\phi& \Tr\\
(vi)& \Kw_i\phi\to\Kw_i(\Kw_i\phi\vee\psi)& (iv),(v)
\end{array}\]
\item Similar to~\ref{prop.kw4}, by using Axiom $\Euc$.
\end{enumerate}
\end{proof}

Before the completeness results, we first show two negative results which demonstrate that only adding $\Tr$ (resp. $\Euc$) on top of $\SPLKw$ is not enough for completeness of $\PLKw$ over transitive (resp. Euclidean) frames. 
\begin{proposition} \label{incompleteness4}
$\SPLKw + \Tr$ is incomplete with respect to the class of transitive frames.
\end{proposition}

\begin{proof}
Recall that $\Kw_ip\to\Kw_i(\Kw_ip\vee q)$ is an instance of $\KwTr$ and it is valid on the class of all transitive frames (Prop.\ref{validthree}). We will show that this formula is not a theorem of $\SPLKw + \Tr$. For this, we construct a model $\M$ such that $\SPLKw + \Tr$ is sound with respect to validity on $\M$ (i.e. for any $\PLKw$ formula $\phi$, $\vdash_{\SPLKw + \Tr}\phi$ implies $\M\vDash\phi$), but $\M\nvDash\Kw_ip\to\Kw_i(\Kw_ip\vee q)$. Therefore $\nvdash_{\SPLKw + \Tr}\Kw_ip\to\Kw_i(\Kw_ip\vee q).$ Since $\Kw_ip\to\Kw_i(\Kw_ip\vee q)$ is not provable in $\SPLKw+\Tr$ but it is valid over transitive frames, $\SPLKw+\Tr$ is not complete w.r.t. the class of all transitive frames. 

Consider the following model $\M$ (w.l.o.g. let us assume $\BP=\{p,q\}$):
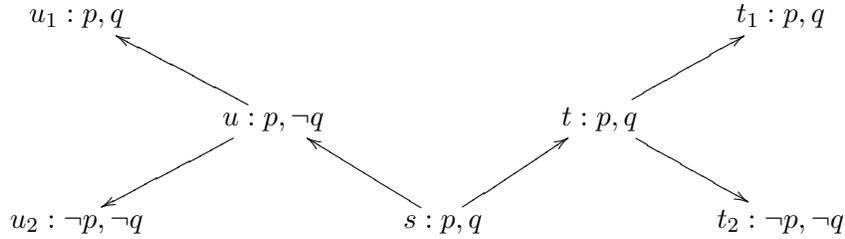
$$
\xymatrix{{u_1: p,q}&&&&{t_1: p,q}\\
&{u:p,\neg q}\ar[ul]\ar[dl]&&{t: p,q}\ar[ur]\ar[dr]& \\
{u_2: \neg p,\neg q}&& {s: p,q} \ar[ur]\ar[ul]   & & {t_2: \neg p,\neg q}
}
$$

\medskip

First, remember that all the axioms of $\SPLKw$ are valid on the class of all frames (Prop.\ref{prop.sound}), thus they are also valid on $\M$. As for the inference rules,  their validities on $\M$ do not follow immediately from the fact that these rules are valid on the class of all frames. However, it is not hard to check that $\MP, \GENKw$ and $\REKw$ are indeed valid on $\M$, i.e., if the premise is valid on $\M$ then the conclusion is also valid on $\M$. 

Second, $\Tr$ is valid on $\M$: by the construction of $\M$, it is not hard to show by induction on the structure of $\phi\in\PLKw$ that: for any $\phi$, $t_1\vDash\phi$ iff $u_1\vDash\phi$, and $t_2\vDash\phi$ iff $u_2\vDash\phi$ $(\ast)$. As none of worlds $t_1,t_2,u_1,u_2$ has any successor, then all of them satisfy $\Kw_i\Kw_i\phi$, thus also satisfy $\Tr$ ($\Kw_i\phi\to\Kw_i\Kw_i\phi$). Also, since $t_1$ and $t_2$ both satisfy $\Kw_i\phi$ for any $\phi$,  $t\vDash\Kw_i\Kw_i\phi$ for any $\phi$ too, and thus $t\vDash\Kw_i\phi\to\Kw_i\Kw_i\phi$. Similarly, we can show that $u\vDash\Kw_i\phi\to\Kw_i\Kw_i\phi$ for any $\phi$. Now from $(\ast)$ we can see  $t\vDash\Kw_i\phi$ iff $u\vDash\Kw_i\phi$, which implies $s\vDash\Kw_i\Kw_i\phi$, and thus $s\vDash\Kw_i\phi\to\Kw_i\Kw_i\phi$. In sum, $\Tr$ is valid on $\M$.

Finally, it is clear that $\M,s\nvDash\Kw_ip\to\Kw_i(\Kw_ip\vee q)$, thus $\M\nvDash\Kw_ip\to\Kw_i(\Kw_ip\vee q)$. Since $\SPLKw+\Tr$ is sound w.r.t. $\M$ then we have $\nvdash_{\SPLKw+\Tr}\Kw_ip\to\Kw_i(\Kw_ip\vee q)$.
\end{proof}

\begin{proposition} \label{incompleteness5}
$\SPLKw + \Euc$ is incomplete with respect to the class of Euclidean frames.
\end{proposition}

\begin{proof}
The strategy is similar to the one in the proof of Theorem \ref{incompleteness4}. Recall that the formula $\neg\Kw_ip\to\Kw_i(\neg\Kw_ip\vee q)$ is valid on the class of all Euclidean frames (Prop.\ \ref{validthree}). We only need to show that this formula is not a theorem of $\SPLKw + \Euc$. For this, we construct a model $\N$ such that $\SPLKw + \Euc$ is sound with respect to $\N$ (i.e., all the theorems of $\SPLKw + \Euc$ are valid on $\N$), but $\N\nvDash\neg\Kw_ip\to\Kw_i(\neg\Kw_ip\vee q)$.

Consider the following model $\N$ (again, let us assume $\BP=\{p,q\}$):

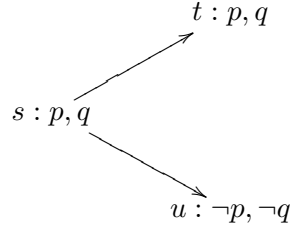
$$
\xymatrix{&{t: p,q}\\
{s: p,q}\ar[ur]\ar[dr]& \\
& {u: \neg p,\neg q}
}
$$

\medskip

As in the previous proof, the axioms and inference rules of $\SPLKw$ are valid on $\N$.

Now we show $\Euc$ is valid on $\N$: by the construction of $\N$, neither $t$ nor $u$ has successor, then they both satisfy $\Kw_i\neg\Kw_i\phi$, and thus satisfy $\Euc$ ($\neg\Kw_i\phi\to\Kw_i\neg\Kw_i\phi$). Also, $t\vDash\Kw_i\phi$ and $u\vDash\Kw_i\phi$, then $s\vDash\Kw_i\neg\Kw_i\phi$, and thus $s\vDash\neg\Kw_i\phi\to\Kw_i\neg\Kw_i\phi$.

It is clear that $\N,s\nvDash\neg\Kw_ip\to\Kw_i(\neg\Kw_ip\vee q)$, thus $\N\nvDash\neg\Kw_ip\to\Kw_i(\neg\Kw_ip\vee q)$.
\end{proof}

\weg{\begin{proof}
We show that $\KwEuc$ is not a derivable theorem in $\SPLKw + \Euc$. The proof is similar to that of Proposition \ref{incompleteness4}, only now we use world $t$ in $\M$. We first show that all instantiations of $\Euc$ are true in $t$: for all $\phi$, $t_1\vDash\Kw_i\phi$ and $t_2\vDash\Kw_i\phi$ (there are no successors). Therefore, $t\vDash\Kw_i\Kw_i\phi$, but we also have $t\vDash\Kw_i\neg\Kw_i\phi$, and so $t\vDash\neg\Kw_i\phi\to\Kw_i\neg\Kw_i\phi$. Therefore, similarly to above, all theorems of $\SPLKw + \Euc$ are true on $t$. But $\M,t\nvDash\neg\Kw_ip\to\Kw_i(\neg\Kw_ip\vee q)$. Therefore not all instantiations of $\KwEuc$ are true in $t$. And thus we have shown that $\KwEuc$ is not a derivable axiom in $\SPLKw + \Euc$.
\end{proof}}

\weg{
\begin{proof}
\begin{enumerate}
\item The following is a derivation in $\SPLKwTTr$:
\[ \begin{array}{lll}
(i)& \Kw_i\phi\to(\Kw_i\phi\vee\psi)& \TAUT\\
(ii)& \Kw_i(\Kw_i\phi\to(\Kw_i\phi\vee\psi))& \GENKw,(i)\\
(iii)& \Kw_i\Kw_i\phi\land \Kw_i(\Kw_i\phi\to(\Kw_i\phi\vee\psi))\land \Kw_i\phi\to \Kw_i(\Kw_i\phi\vee\psi)& \KwT\\
(iv)& \Kw_i\Kw_i\phi\land \Kw_i\phi\to \Kw_i(\Kw_i\phi\vee\psi)& (ii),(iii)\\
(v)& \Kw_i\phi\to\Kw_i\Kw_i\phi& \Tr\\
(vi)& \Kw_i\phi\to\Kw_i(\Kw_i\phi\vee\psi)& (iv),(v)
\end{array}\]
\item Similar to~\ref{prop.kw4}, by using Axiom $\Euc$.
\end{enumerate}
\end{proof}
}

\medskip

We now continue with the completeness proofs for the extended proof systems. We first address the completeness of $\SPLKwT$. In the canonical model construction of Def.~\ref{cononicalmodel} it is unclear whether the canonical relation is reflexive. To ensure that the relations are reflexive, we take the reflexive closure of the canonical relation.

\begin{definition}[Canonical model of \SPLKwT]\label{def.canonicalmodelT}
The canonical model $\M^c$ of \SPLKwT~ is the same as $\M^c$ in Def.~\ref{cononicalmodel}, except that $S^c$ consists of all maximal consistent sets \emph{of $\SPLKwT$}, and that $\to^c_i$ is \emph{the reflexive closure of} the canonical relation defined in Def.~\ref{cononicalmodel}.
\end{definition}
As before, we use an equivalent definition of the canonical relation: $s\to^c_it$ iff $s=t$ or (there exists a $\chi$ such that $\neg \Kw_i\chi\in s$, and for all $\phi$ and $\psi$: $\phi\in t$ implies that at least one of $\Kw_i\phi$, $\Kw_i(\phi\to\psi)$ and $\neg \Kw_i\psi$ is not in $s$).

\begin{lemma}[Truth Lemma for \SPLKwT]\label{truthlemt}
For any $\PLKw$~ formula $\phi$, $\M^c,s\vDash\phi$ iff $\phi\in s$.
\end{lemma}
\begin{proof}
By induction on $\phi$. We consider the non-trivial case for $\Kw_i\phi$.

\emph{Left-to-right}: This is similar to the corresponding proof in Lemma \ref{truthlem}. Observe that all pairs in the canonical relation in Def.~\ref{cononicalmodel} are also in the relation $\to^c_i$ from Def.~\ref{def.canonicalmodelT}.

\emph{Right-to-left}:
Assume towards contradiction that $\Kw_i\phi\in s$ but $\M^c,s\vDash\neg\Kw_i\phi$, namely there are distinct states $t_1$ and $t_2$ such that $s\to^c_it_1$ and $s\to^c_it_2$ and $\M,t_1\vDash\phi$ and $\M,t_2\vDash\neg\phi$. By induction hypothesis, $\phi\in t_1$ and $\neg\phi\in t_2$. As $\to^c_i$ is reflexive, we need to consider two cases ($s=t_1$ and $s=t_2$ is impossible, because $t_1\neq t_2$):
\begin{itemize}
\item $s\neq t_1$ and $s\neq t_2$.\\  Then the proof is same as the corresponding one in Lemma \ref{truthlem}. And finally we can get a contradiction.
\item $s=t_1$ or $s=t_2$. \\ We may as well consider the case $s=t_1$ and $s\neq t_2$. Since $\neg\phi\in t_2$, $\Kw_i\phi\in s$, and $s\to^c_it_2$, by the equivalent definition of $\to^c$, there exists $\chi$ such that $\neg\Kw_i\chi\in s$ and $\neg\Kw_i(\neg\phi\to\chi)\in s$. By $\phi\in s$ and $\Kw_i\phi\in s$ and Axiom $\KwT$, we get $\Kw_i(\phi\to\chi)\to\Kw_i\chi\in s$. Since $\neg\Kw_i\chi\in s$, $\neg\Kw_i(\phi\to\chi)\in s$. Now $\neg\Kw_i(\neg\phi\to\chi)\in s$ and $\neg\Kw_i(\phi\to\chi)\in s$, thus by Axiom \KwDis~ we conclude that $\neg\Kw_i\phi\in s$. Contradiction.
\end{itemize}
\end{proof}

Based on the above lemma, it is routine to show the following.
\begin{theorem}\label{compextensionT}
$\SPLKwT$ is complete with respect to the class of all $\mathcal{T}$-frames.
\end{theorem}

Now let us look at the completeness for \SPLKwTr\ and \SPLKwEuc. In these cases we do not need to revise the canonical relations.
\begin{theorem}\label{compextension4}
\SPLKwTr\ is complete with respect to the class of all $4$-frames.
\end{theorem}
\begin{proof}
Define $\M^c$ as in Def.~\ref{cononicalmodel} w.r.t.\ \SPLKwTr. We only need to show that $\to^c_i$ is transitive.

Given $s,t,u\in S^c$. Assume that $s\to^c_it$ and $t\to^c_iu$. From $s\to^c_it$ it follows that there exist $\chi$ such that $\neg\Kw_i\chi\in s$. To show $s\to^c_iu$, by the definition of the canonical relation, we need to prove that for all $\phi$ and $\psi$: $\Kw_i\phi\land\Kw_i(\phi\to\psi)\land\neg\Kw_i\psi\in s$ implies $\neg\phi\in u$. From now on, let us fix two formulas $\phi$ and $\psi$ such that $\Kw_i\phi\land\Kw_i(\phi\to\psi)\land\neg\Kw_i\psi\in s$. We need to show $\neg\phi\in u$. From $t\to_i^cu$ it follows that there is a $\chi'$ such that $\neg\Kw_i\chi'\in t$. Now according to the definition of $\to^c_i$ again, if we can show that $\Kw_i\phi\land\Kw_i(\phi\to\chi')\in t$, then by $t\to_i^cu$, we have $\neg\phi\in u$.

We first show that $\Kw_i\phi\in t$: As $\Kw_i\phi\in s$, first, by $\Tr$ and \EquiKw\, we get $\Kw_i\neg\Kw_i\phi\in s$; second, by Axiom $\KwTr$ we get $\Kw_i(\Kw_i\phi\vee\psi)\in s$, i.e. $\Kw_i(\neg\Kw_i\phi\to\psi)\in s$. Now we obtain $\Kw_i\neg \Kw_i\phi\in s$, $\Kw_i(\neg\Kw_i\phi\to\psi)\in s$, and $\neg\Kw_i\psi\in s$. By $s\to^c_it$ we have $\Kw_i\phi\in t$.

We now show that $\Kw_i(\phi\to\chi')\in t$: Since $\Kw_i(\phi\to\psi)\land\neg\Kw_i\psi\in s$, it follows from Axiom $\KwCon$ that $\neg\Kw_i(\neg\phi\to\psi)\in s$. Since $\Kw_i\phi\in s$, $\Kw_i(\phi\to\chi')\vee\Kw_i(\neg\phi\to\psi)\in s$ by \KwDis, thus $\Kw_i(\phi\to\chi')\in s$. Using $\Tr$ and \EquiKw, we get $\Kw_i\neg\Kw_i(\phi\to\chi')\in s$; with Axiom $\KwTr$, we get $\Kw_i(\Kw_i(\phi\to\chi')\vee\psi)\in s$, i.e. $\Kw_i(\neg\Kw_i(\phi\to\chi')\to\psi)\in s$. Now $\Kw_i\neg\Kw_i(\phi\to\chi')\land\Kw_i(\neg\Kw_i(\phi\to\chi')\to\psi)\land\neg\Kw_i\psi\in s$. From $s\to^c_it$ we conclude that $\Kw_i(\phi\to\chi')\in t$, as desired.
\end{proof}
\begin{theorem}\label{compextension5}
\SPLKwEuc\ is complete with respect to the class of all $5$-frames.
\end{theorem}
\begin{proof}
Define $\M^c$ as in Def.~\ref{cononicalmodel} w.r.t. $\SPLKwEuc$. We only need to show that $\to^c_i$ is Euclidean. Let $s,t,u\in S^c$ be given, and assume $s\to^c_it$ and $s\to^c_iu$. We then need to show that $t\to^c_iu$, that is to say:
\begin{itemize} \item There exists a $\chi$ such that $\neg\Kw_i\chi\in t$ \hfill $(\dag)$ \item For all $\phi$ and $\psi$: $\Kw_i\phi\land\Kw_i(\phi\to\psi)\land\neg\Kw_i\psi\in t$ implies $\neg\phi\in u$. \hfill $(\ddag)$
\end{itemize}

$(\dag)$: from the assumption $s\to^c_it$ it follows that there exists $\chi_1$ such that $\neg\Kw_i\chi_1\in s$. In the following we prove $\neg\Kw_i\chi_1\in t$. Using $s\to^c_it$ again, if we can show $\Kw_i\Kw_i\chi_1\land\Kw_i(\Kw_i\chi_1\to\chi_1)\land\neg\Kw_i\chi_1\in s$ then we are done. By $\neg\Kw_i\chi_1\in s$, $\Euc$ and \EquiKw, we have $\Kw_i\Kw_i\chi_1\in s$. Using $\neg\Kw_i\chi_1\in s$ again and $\vdash\neg\Kw_i\chi_1\to\Kw_i(\neg\Kw_i\chi_1\vee\chi_1)$ (an instance of Axiom $\KwEuc$), we get $\Kw_i(\neg\Kw_i\chi_1\vee\chi_1)\in s$, and thus $\Kw_i(\Kw_i\chi_1\to\chi_1)\in s$ from $\TAUT$ and $\REKw$. Hence $\Kw_i\Kw_i\chi_1\land\Kw_i(\Kw_i\chi_1\to\chi_1)\land\neg\Kw_i\chi_1\in s$, and therefore $\neg\Kw_i\chi_1\in t$ by $s\to^c_it$.

$(\ddag)$: Now fixing two formulas $\phi,\psi$ such that $\Kw_i\phi\land\Kw_i(\phi\to\psi)\land\neg\Kw_i\psi\in t$, we need to show $\neg\phi\in u$. By the similar strategy as in the proof of Thm.~\ref{compextension4}, if we can prove $\Kw_i\phi\land\Kw_i(\phi\to\chi_1)\in s$,
then by $\neg\Kw_i\chi_1\in s$ and $s\to^c_iu$, we get $\neg\phi\in u$.

We first show $\Kw_i\phi\in s$: if not, i.e.\ $\neg\Kw_i\phi\in s$, then first, by $\Euc$ and \EquiKw\ we get $\Kw_i\Kw_i\phi\in s$; second, by Axiom $\KwEuc$ we get $\Kw_i(\neg\Kw_i\phi\vee\chi_1)\in s$, i.e. $\Kw_i(\Kw_i\phi\to\chi_1)\in s$. Remember that $\neg\Kw_i\chi_1\in s$, thus we have shown $\Kw_i\Kw_i\phi\land\Kw_i(\Kw_i\phi\to\chi_1)\land\neg\Kw_i\chi_1\in s$. Then, by $s\to^c_it$ we get $\neg\Kw_i\phi\in t$, a contradiction.

We now show $\Kw_i(\phi\to\chi_1)\in s$: if not, i.e. $\neg\Kw_i(\phi\to\chi_1)\in s$, then first, by $\Euc$ and \EquiKw\ we get $\Kw_i\Kw_i(\phi\to\chi_1)\in s$; second, by Axiom $\KwEuc$ we get $\Kw_i(\neg\Kw_i(\phi\to\chi_1)\vee\chi_1)\in s$, i.e. $\Kw_i(\Kw_i(\phi\to\chi_1)\to\chi_1)\in s$. Now $\Kw_i\Kw_i(\phi\to\chi_1)\land\Kw_i(\Kw_i(\phi\to\chi_1)\to\chi_1)\land\neg\Kw_i\chi_1\in s$, then by $s\to^c_it$ we get $\neg\Kw_i(\phi\to\chi_1)\in t$. Moreover, by supposition $\Kw_i(\phi\to\psi)\land\neg\Kw_i\psi\in t$ and Axiom $\KwCon$ we derive $\neg\Kw_i(\neg\phi\to\psi)\in t$. Then from  $\neg\Kw_i(\phi\to\chi_1)\land\neg\Kw_i(\neg\phi\to\psi)\in t$ and Axiom $\KwDis$ it follows that $\neg\Kw_i\phi\in t$, contradiction.
\end{proof}

\begin{corollary}\label{compextension45}
\SPLKwTrEuc\ is complete with respect to the class of all $45$-frames.
\end{corollary}
\begin{proof}
This follows directly from Thm.~\ref{compextension4} and Thm.~\ref{compextension5}. The canonical model w.r.t. \SPLKwTrEuc\ is both transitive and Euclidean.
\end{proof}

\begin{theorem}\label{compextensionS4}
$\SPLKwTTr$ is complete with respect to the class of all $\mathcal{S}4$-frames.
\end{theorem}

\begin{proof}
Define $\M^c$ as Def.~\ref{def.canonicalmodelT} w.r.t. $\SPLKwTTr$. Given \weg{the reflexive definition of $\to^c_i$ and }Thm.~\ref{compextensionT},\weg{The reflexivity of $\to^c_i$ is immediate from Thm.~\ref{compextensionT}.} we only need to show that $\to^c_i$ is transitive. Now given $s,t,u\in S^c$, and assume $s\to^c_it$ and $t\to^c_iu$, we need to show $s\to^c_iu$. If $s=t$ or $t=u$ or $s=u$, then by the assumption and the fact that $\to^c_i$ is reflexive, we get $s\to^c_iu$. Thus we consider the case $s\neq t, t\neq u$ and
$s\neq u$. The proof for this case is the same as Thm.~\ref{compextension4}, as we can use $\KwTr$ due to Prop.\ \ref{prop.kw4kw5}.
\end{proof}
\begin{theorem}\label{compextensionS5}
$\SPLKwTEuc$ is complete with respect to the class of all $\mathcal{S}5$-frames.
\end{theorem}

\begin{proof}
Define $\M^c$ as Def.~\ref{def.canonicalmodelT} w.r.t. $\SPLKwTEuc$. Given Thm.~\ref{compextensionT}, we only need to show that $\to^c_i$ is Euclidean. Now given $s,t,u\in S^c$, and assume $s\to^c_it$ and $s\to^c_iu$, we need to show $t\to^c_iu$. If $s\neq t$, $t\neq u$ and $s\neq u$, then the proof is the same as in Thm.~\ref{compextension5}, as we can use $\KwEuc$ due to Prop.\ \ref{prop.kw4kw5}. If $s=t$ or $t=u$, then by the assumption and the fact that $\to^c_i$ is reflexive, we get $t\to^c_iu$. If $s=u$ and $t\not=u$, we need to show $t\to^c_iu$. This can be proved by using Axiom $\KwT$ instead of the assumption that $s\to^c_iu$ in the corresponding proof of Thm.~\ref{compextension5}.
\end{proof}

\section{Knowing whether logic with announcements} \label{sec.announcement}

In the muddy children puzzle, children learn their status by repeating the announcement ``nobody knows whether he or she is muddy.''
 In this section we add public announcement modalities to knowing whether logic. We will first give the language and its semantics, and then propose an axiomatization that can be shown to be complete because all formulas with announcements are provably equivalent to formulas without announcement (the proof system defines a rewrite procedure).

\begin{definition}[Language $\PLKwA$]
The language of $\PLKwA$ is obtained by adding an inductive clause $[\phi]\phi$ to the construction of the language \PLKw\ (see Def.~\ref{def.language}).
\end{definition}
The formula $[\phi]\psi$ says that ``after public announcement of $\phi$, $\psi$ holds.'' 

\begin{definition}\label{semanticsofannouncement}
Let $\mc{M} = \langle S,\toallp, V\rangle$ be a model and $\phi,\psi\in\PLKwA$. The semantics of public announcement is as follows.
\[
\mc{M},s\vDash[\phi]\psi \ \ \Leftrightarrow \ \ \mc{M},s\vDash \phi \textrm{ implies } \mc{M}|_{\phi},s\vDash \psi
\]
where $\mc{M}|_{\phi}=\langle S',\toallp, V'\rangle$ is such that $S'=\{s\in S\mid \M,s\vDash\phi\}$, ${\to'_i}={\to_i}\cap{(S'\times S')}$, and $V'(p)=V(p)\cap S'$.
\end{definition}

Unlike $\PLKw$, the logic $\PLKwA$ is not closed under uniform substitution. For instance, $p\to[q]p$ is valid, but $\neg\Kw_iq \to [q]\neg\Kw_iq$ is not valid, as demonstrated by the following example, wherein $\M,s \not\models \neg\Kw_iq \to [q]\neg\Kw_iq$.
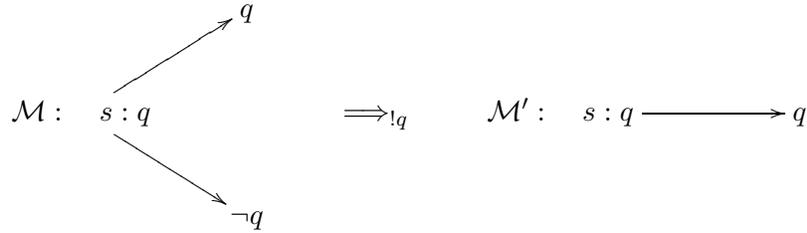
$$
\xymatrix{&{q}& &&\\
\M: \ \ \ \  {s:q} \ar[ur]\ar[dr]   & & \Longrightarrow_{!q} & \M': \ \ \ \  {s:q} \ar[rr]&&{q} \\
&{\neg q}  && &           }
$$
This is the reason that the proof system below must contain formula variables (schematic formulas) instead of propositional variables, and also for that reason we have presented the proof system $\SPLKw$ in the same way.

\begin{definition}[Proof system \SPLKwA] \label{def.splkwa}
The proof system $\SPLKwA$ is the extension of $\SPLKw$ (Def.~\ref{axiomstable}) with the following reduction axioms \weg{and rules} for announcements.
\[
\begin{array}{ll}
\ATOM&[\phi]p\lra(\phi\to p)\\
\NEG&[\phi]\neg\psi\lra(\phi\to\neg[\phi]\psi)\\
\CCOM&[\phi](\psi\land\chi)\lra([\phi]\psi\land[\phi]\chi)\\
\AAA&[\phi][\psi]\chi\lra[\phi\land[\phi]\psi]\chi\\
\AKw&[\phi]Kw_i\psi\lra(\phi\to(\Kw_i[\phi]\psi\vee\Kw_i[\phi]\neg\psi))\\

\end{array}
\]
\end{definition}

\begin{proposition}[Soundness]
\SPLKwA~is sound with respect to the class of all frames.
\end{proposition}

\begin{proof}
We only consider the non-trivial axiom schema $\AKw$.

Left-to-right: Given any model $\M=\lr{S,\toall,V}$ based on a frame and $s\in S$, assume that $\M,s\vDash[\phi]\Kw_i\psi$. We now need to show that $\M,s\vDash\phi\to(\Kw_i[\phi]\psi\vee\Kw_i[\phi]\neg\psi)$. For this, suppose $\M,s\vDash\phi$, to show $\M,s\vDash\Kw_i[\phi]\psi\vee\Kw_i[\phi]\neg\psi$. By reductio ad absurdum we suppose $\M,s\nvDash\Kw_i[\phi]\psi\vee\Kw_i[\phi]\neg\psi$. Then $\M,s\nvDash\Kw_i[\phi]\psi$ and $\M,s\nvDash\Kw_i[\phi]\neg\psi$. That is to say, there exist $t,t'\in S$ such that $s\to_it, s\to_it'$ and $t\vDash[\phi]\psi,t'\vDash\neg[\phi]\psi$ and, there exist $u,u'\in S$ such that $s\to_iu,s\to_iu'$ and $u\vDash[\phi]\neg\psi,u'\vDash\neg[\phi]\neg\psi$. It follows that $\M,t'\vDash\phi$ and $\M|_\phi,t'\vDash\neg\psi$ from $t'\vDash\neg[\phi]\psi$, and $\M,u'\vDash\phi$ and $\M|_\phi,u'\vDash\psi$ from $u'\vDash\neg[\phi]\neg\psi$, where $\M|_\phi$ is defined as Definition \ref{semanticsofannouncement}. Moreover, we have $s\to'_it',s\to'_iu'$ in $\M|_{\phi}$ because $\M,s\vDash\phi,\M,t'\vDash\phi,\M,u'\vDash\phi$ and $s\to_it',s\to_iu'$. Then $\M|_\phi,s\nvDash\Kw_i\psi$, contradicting the assumption $\M,s\vDash[\phi]\Kw_i\psi$ and $\M,s\vDash\phi$.

Right-to-left: Assume $\M,s\vDash\phi$. First consider the case that $\M,s\vDash\Kw_i[\phi]\psi$. Then, either for all $t$ with $s\to_it$ we have $\M,t\vDash[\phi]\psi$ or for all $t$ with $s\to_it$ we have $\M,t\vDash\neg[\phi]\psi$. In the first case, with $\M,s\vDash\phi$, we get for all $t$ with $s\to'_it$, $\M|_{\phi},t\vDash\psi$. In the second case, with $\M,s\vDash\phi$, we get for all $t$ with $s\to'_it$, $\M|_{\phi},t\vDash\neg\psi$. In either subcase we both get $\M|_{\phi},s\vDash\Kw_i\psi$. Now consider the case that $\M,s\vDash\Kw_i[\phi]\neg\psi$. Similarly, in this case we can also get $\M|_{\phi},s\vDash\Kw_i\psi$. Therefore we can conclude that $\M,s\vDash[\phi]\Kw_i\psi$.
\end{proof}

The logic $\PLKwA$ is equally expressive as knowing whether logic, as the axiomatization induces a rewrite procedure. By defining a suitable complexity, we can rewrite every formula in $\PLKwA$ as a logically equivalent formula of $\PLKw$ of lower complexity, and thus the completeness for $\SPLKwA$ follows from the completenes of $\SPLKw$ (see \cite{hvdetal.del:2007,WC13} for this reduction technique).

\begin{theorem} [Completeness of \SPLKwA]
For every $\phi\in\PLKwA$, $\vDash\phi$ implies~ $\vdash\phi$.
\end{theorem}

As axiomatization \SPLKwA\ gives a translation of $\PLKwA$ into $\PLKw$, and $\PLKw$ is decidable (Prop.~\ref{deciableknowingwhether}), the logic of knowing whether with announcements is also decidable.
\begin{proposition} $\PLKwA$ is decidable.
\end{proposition}

We can also consider the logic of knowing whether with announcement on other frame classes, where our main interest is the class of $\mathcal{S}5$ frames. The expressivity of knowing whether logics for other frame classes also does not change by adding the announcement operator, as the reduction axioms and rules still allow every formula to be rewritten to an equivalent expression without announcements (so, a fortiori, this also holds for theorems of those logics).
\begin{theorem} \label{moreresult}
Consider the proof system $\SPLKwATEuc$ that extends $\SPLKwA$ with $\KwT$ and $\Euc$. $\SPLKwATEuc$ is complete with respect to the class of $\mathcal{S}5$-frames.
\end{theorem}

\section{Comparison with the literature} \label{sec.comparison}

In \cite{wiebeetal:2003,hoeketal:2004}, the authors give a complete axiomatization of a logic of ignorance with primitive modal construct $I \phi$, for `the agent is ignorant about $\phi$'. If an agent is ignorant about $\phi$, she does not know whether $\phi$, so $I \phi$ is definable as $\neg \Kw \phi$. Their axiomatization $\SIg$ is shown in Def.~\ref{def.sig}, wherein we have replaced $I$ by $\neg \Kw$. It is different from ours. Now it is of course a matter of taste whether one prefers the system $\SPLKw$ over $\PLKw$ (page \pageref{axiomstable}) or the one below, but we tend to find ours simpler, e.g.\ with respect to the axioms {\tt I3} and {\tt I4} below.

\begin{definition}[Axiomatization $\SIg$ \cite{wiebeetal:2003,hoeketal:2004}] \label{def.sig}
\[ \begin{array}{ll}
\texttt{I0}& \text{All instances of propositional tautologies} \\
\texttt{I1}&\neg\Kw_i\phi\lra\neg\Kw_i\neg\phi\\
\texttt{I2}& \neg\Kw_i(\phi\land\psi)\to\neg\Kw_i\phi\vee\neg\Kw_i\psi \\
\texttt{I3}& (\Kw_i\phi\land\neg\Kw_i(\chi_1\land\phi)\land\Kw_i(\phi\to\psi)\land\neg\Kw_i(\chi_2\land(\phi\to\psi)))\to\Kw_i\psi\land\neg\Kw_i(\chi_1\land\psi)\\
\texttt{I4}& \Kw_i\psi\land\neg\Kw_i\chi\to\neg\Kw_i(\chi\land\psi)\vee\neg\Kw_i(\chi\land\neg\psi)\\
\texttt{RI} & \text{From } \phi \text{ infer } \Kw_i\phi\land(\neg\Kw_i\chi\to\neg\Kw_i(\chi\land\phi)) \\
\texttt{MP} & \text{Modus Ponens}\\
\texttt{Sub}& \text{Substitution of equivalents}
\end{array} \]
\end{definition}

Since both systems are complete, their axioms and inference rules are derivable in our system $\SPLKw$, and we show precisely how to do it: i.e., we will derive in $\SPLKw$ axioms {\tt I2}, {\tt I3}, and {\tt I4}, and the rules {\tt RI} and \RE. This lengthy exercise is reported in Appendix \ref{sec.appendixa}. $\SPLKw$ can also be derived from ${\bf Ig}$ due to the completeness of ${\bf Ig}$.

\begin{proposition} \label{prop.zxcvzxcv}
All the axioms of $\SIg$ are derivable in $\SPLKw$ and all the rules of ${\bf Ig}$ are admissible in $\SPLKw$, and vice versa.
\end{proposition}

The proof system $\SIg$ is also extended with an axiom $\texttt{G4}$, which we present in terms of $\Kw$:
$$\neg\Kw_i\chi \ \to \ \big( \ \Kw_i\phi\land\neg\Kw_i(\phi\land\chi) \ \to \ \Kw_i(\Kw_i\phi\land\neg\Kw_i(\chi\land\phi))\land\neg\Kw_i(\Kw_i\phi\land\neg\Kw_i(\phi\land\chi)\land\chi) \ \big)$$
It is then claimed that $\SIg + \texttt{G4}$ is a complete axiomatization of the logic of ignorance over transitive frames \cite[Lemma 4.2]{hoeketal:2004}. Unfortunately, we think that $\texttt{G4}$ is invalid, thus the system is not sound. Consider this countermodel $\mathcal{M}$
$$
\xymatrix{&&{t_1: p,q}\\
&{t: p,q}\ar[ur]\ar[dr]& \\
\M: \ \  \ \ \ {s:p,q} \ar[ur]\ar[dr]\ar@/^30pt/[uurr]\ar@/_/[rr]   & & {t_2: \neg p,q} \\
&{u: \neg p, q}  &          }
$$
and the formula $$\neg\Kw_ip\to(\Kw_iq\land\neg\Kw_i(q\land p)\to\Kw_i(\Kw_iq\land\neg\Kw_i(p\land q))\land\neg\Kw_i(\Kw_iq\land\neg\Kw_i(q\land p)\land p))$$
Observe $s\vDash\neg\Kw_ip$ and $s\vDash\Kw_iq \land \neg\Kw_i(q\land p)$. Then, note that $s\not\vDash \Kw_i(\Kw_iq\land\neg\Kw_i(p\land q))$ (take $u$ and $t$ as two witnesses), thus $s\not\vDash \Kw_i(\Kw_iq\land\neg\Kw_i(p\land q))\land\neg\Kw_i(\Kw_iq\land\neg\Kw_i(q\land p)\land p)$. Therefore, this formula is {\em false} in state $s$ of this model $\M$, which invalidates $\texttt{G4}$.\footnote{Confirmed by the authors of \cite{hoeketal:2004} by personal communication.}

In this paper, we advanced the research beyond \cite{hoeketal:2004} by proving expressivity results and more undefinability results. And more importantly, apart from correctly axiomatizing knowing whether logic over transitive frames (the system $\SPLKwTr$), we also axiomatized \PLKw\ on various other frame classes, which was considered hard in \cite{hoeketal:2004}. Further, we extended knowing whether logic with public announcements, and gave a complete axiomatization for that extension.

\medskip

Another recent work on a logic of ignorance is \cite{steinsvold:2008}. The author gives a topological semantics for the logic of ignorance and completely axiomatizes it by the following proof system ${\bf LB}$ (we have replaced $\Box$ in \cite{steinsvold:2008} by $\Kw$):

\begin{definition}[Axiomatization ${\bf LB}$]\label{def.slb}
\[ \begin{array}{ll}
\TAUT& \text{All instances of propositional tautologies}\\
\texttt{N}& \Kw_i\top\lra\top \\
\texttt{Z}&\Kw_i\phi\lra\Kw_i\neg\phi\\
\texttt{R}& \Kw_i\phi\land\Kw_i\psi\to\Kw_i(\phi\land\psi) \\
\texttt{WM} & \text{From } \Kw_i\phi\land\phi\to\psi \text{ infer } \Kw_i\phi\land\phi\to\Kw_i\psi\land\psi\\
\texttt{MP} & \text{Modus Ponens}\\
\texttt{Sub}& \text{Substitution of equivalents}
\end{array} \]
\end{definition}

This proof system is equivalent to our system $\SPLKwTTr$ for \PLKw\ over $\mathcal{S}4$-frames in the following sense.
\begin{proposition} \label{prop.qwerqwer}
All the axioms of ${\bf LB}$ are derivable in $\SPLKwTTr$ and all the rules of ${\bf LB}$ are admissible in $\SPLKwTTr$, and vise versa.
\end{proposition}
Appendix~\ref{sec.appendixb} contains the proof. Unlike Proposition \ref{prop.zxcvzxcv}, Proposition \ref{prop.qwerqwer} cannot be obtained using the completeness of both systems, since the semantics of the two logics are different. Compared to $\SPLKwTTr$, the axioms of ${\bf LB}$ are simpler, while the rules are more complicated (${\tt WM}$ is clearly a complex derivation rule, and in $\SPLKwTTr$ the rule $\RE$ is admissable instead). It is again a matter of taste which system is preferable. Nevertheless, the above result also shows that the topological semantics in \cite{steinsvold:2008} is equivalent to our Kripke semantics over $\mathcal{S}4$-frames, modulo validity.

\section{Conclusions and further research} \label{sec.conclusions}

We first summarize our contributions.
\begin{itemize}
\item We gave complete axiomatizations of \PLKw\ over the frame classes $\mathcal{K}, \mathcal{T}, 4, 5, 45, \mathcal{S}4$, and $\mathcal{S}5$.
\item $\PLKw$ cannot define the frame classes $\mathcal{D}, \mathcal{T}, \mathcal{B}, 4$, and $5$.
\item $\PLKw$ is less expressive than \EL\ over model classes $\mathcal{K}, \mathcal{D}, \mathcal{B}, 4$, and $5$. It is equally expressive as \EL\ over $\mathcal{T}$ (and classes contained in $\mathcal{T}$, such as $\mathcal{S}4$ and $\mathcal{S}5$).
\item We axiomatized the logic of knowing whether with public announcements, \PLKwA, and demonstrated that it is equally expressive as $\PLKw$.
\item The axiomatization \SPLKw\ for knowing whether logic is equivalent to \textbf{Ig} of \cite{wiebeetal:2003}, and the axiomatization $\SPLKwTTr$ for knowing whether logic over transitive frames is equivalent to \textbf{LB} of \cite{steinsvold:2008}.
\end{itemize}

\noindent We continue with ideas on further research.
\begin{itemize}
\item To complete the axiomatization spectrum, we want to axiomatize \PLKw\ over $\mathcal{D}$- and $\mathcal{B}$-frames. We expect similar techniques as in the case of \SPLKwT\ to work, while finding the right axioms may be hard.

\item As said, knowing whether seems a natural modality and able to express statements succinctly. To make this intuition concrete,  we conjecture that \PLKw\ over reflexive models is exponentially more succinct than \EL\ if there are at least two agents. The computational complexity of knowing whether logics is also left for future work.

\item The comparison with \cite{steinsvold:2008} demonstrates that the same logic may be obtained by different semantics based on different models. The undefinability of frame properties suggests that the Kripke semantics may not be the best semantics for knowing whether logic. We intend to investigate neighbourhood semantics and other weaker semantics for \PLKw.

\item We consider adding group operators for knowing whether (or ignorance) to the language. There are various options to define such group operators. Is a group $G$ ignorant of $\phi$ if, when defining the accessibility relation for $G$ as the transitive closure of the union of all relations, both a state where $\phi$ is true and a state where $\phi$ is false are group-accessible? Or should all agents consider states possible where $\phi$ is true and where $\phi$ is false, and then we `simply' take Kleene-iteration of that? There are yet other ways to define group ignorance, and the notion of group ignorance is under close scrutiny in formal epistemology \cite{hansen:2011,hendricks:2010}.

\item We consider adding arbitrary announcement operators \cite{balbianietal:2008} to knowing whether logic. One can then express, for example, that after any announcement agent $i$ remains ignorant: $\Box \neg \Kw_i \phi$. This addition becomes more challenging if one then removes the announcement operators from the logical language and defines the arbitrary announcement by modally definable model restrictions.
\end{itemize}

\bibliographystyle{alpha}
\bibliography{biblio2013}

\newcommand{\etalchar}[1]{$^{#1}$}
\begin{thebibliography}{vDvdHK07}

\bibitem[AEJ13]{alonietal:2013}
M.~Aloni, P.~\'Egr\'e, and T.~Jager.
\newblock Knowing whether $\text{A}$ or $\text{B}$.
\newblock {\em Synthese}, 190(14):2595--2621, 2013.

\bibitem[BBvD{\etalchar{+}}08]{balbianietal:2008}
P.~Balbiani, A.~Baltag, H.~van Ditmarsch, A.~Herzig, T.~Hoshi, and T.~De Lima.
\newblock `{K}nowable' as `known after an announcement'.
\newblock {\em Review of Symbolic Logic}, 1(3):305--334, 2008.

\bibitem[Han11]{hansen:2011}
J.U. Hansen.
\newblock {\em A logic toolbox for modeling knowledge and information in
  multi-agent systems and social epistemology}.
\newblock PhD thesis, Roskilde University, 2011.

\bibitem[Hen10]{hendricks:2010}
V.F. Hendricks.
\newblock Knowledge transmissibility and pluralistic ignorance: A first stab.
\newblock {\em Metaphilosophy}, 41:279--291, 2010.

\bibitem[HHL88]{hedetniemietal:1988}
S.M. Hedetniemi, S.T. Hedetniemi, and A.L. Liestman.
\newblock A survey of gossiping and broadcasting in communication networks.
\newblock {\em Networks}, 18:319--349, 1988.

\bibitem[HHS96]{Hart:1996}
S.~Hart, A.~Heifetz, and D.~Samet.
\newblock Knowing whether, knowing that, and the cardinality of state spaces.
\newblock {\em Journal of Economic Theory}, 70(1):249--256, 1996.

\bibitem[HS93]{heifetz1993universal}
A.~Heifetz and D.~Samet.
\newblock {\em Universal Partition Structures}.
\newblock Working paper (Israel Institute of business research). Tel Aviv
  University, 1993.

\bibitem[MDH86]{mosesetal:1986}
Y.O. Moses, D.~Dolev, and J.Y. Halpern.
\newblock Cheating husbands and other stories: a case study in knowledge,
  action, and communication.
\newblock {\em Distributed Computing}, 1(3):167--176, 1986.

\bibitem[Pla89]{plaza:1989}
J.A. Plaza.
\newblock Logics of public communications.
\newblock In {\em Proc.\ of the 4th ISMIS}, pages 201--216. Oak Ridge National
  Laboratory, 1989.

\bibitem[Pla07]{plaza:2007}
J.A. Plaza.
\newblock Logics of public communications.
\newblock {\em Synthese}, 158(2):165--179, 2007.
\newblock Reprint of Plaza's 1989 workshop paper.

\bibitem[Ste08]{steinsvold:2008}
C.~Steinsvold.
\newblock A note on logics of ignorance and borders.
\newblock {\em Notre Dame J.\ Formal Logic}, 49(4):385--392, 2008.

\bibitem[vD07]{hvd.plaza:2007}
H.~van Ditmarsch.
\newblock Comments to `{L}ogics of public communications'.
\newblock {\em Synthese}, 158(2):181--187, 2007.

\bibitem[vdHL03]{wiebeetal:2003}
W.~van~der Hoek and A.~Lomuscio.
\newblock Ignore at your peril - towards a logic for ignorance.
\newblock In {\em Proc.\ of 2nd AAMAS}, pages 1148--1149. ACM, 2003.

\bibitem[vdHL04]{hoeketal:2004}
W.~van~der Hoek and A.~Lomuscio.
\newblock A logic for ignorance.
\newblock {\em Electronic Notes in Theoretical Computer Science},
  85(2)(2):117--133, 2004.

\bibitem[vDvdHK07]{hvdetal.del:2007}
H.~van Ditmarsch, W.~van~der Hoek, and B.~Kooi.
\newblock {\em Dynamic Epistemic Logic}, volume 337 of {\em Synthese Library}.
\newblock Springer, 2007.

\bibitem[WC13]{WC13}
Y.~Wang and Q.~Cao.
\newblock On axiomatizations of public announcement logic.
\newblock {\em Synthese}, 2013.
\newblock Online first: {http://dx.doi.org/10.1007/s11229-012-0233-5}.

\bibitem[WF13]{wangetal:2013}
Y.~Wang and J.~Fan.
\newblock Knowing that, knowing what, and public communication: Public
  announcement logic with {Kv} operators.
\newblock In {\em Proc.\ of 23rd IJCAI}, pages 1147--1154, 2013.

\end{thebibliography}

\appendix

\section{Comparison with \cite{wiebeetal:2003}} \label{sec.appendixa}

We first derive auxiliary theorems that will be used in the derivations of the axioms and rules of the system ${\bf Ig}$.

\begin{lemma}\label{usefulprop1}
$\vdash_\SPLKw\Kw_i(\chi\land\phi)\land\Kw_i(\neg\chi\land\phi)\to\Kw_i\phi$.\\
\end{lemma}

\begin{proof}
$$
\begin{array}{lll}
(i)&(\chi\land\phi)\lra\neg(\chi\to\neg\phi)&\TAUT\\
(ii)&\Kw_i(\chi\land\phi)\lra\Kw_i\neg(\chi\to\neg\phi)&\REKw,(i)\\
(iii)&\Kw_i(\chi\to\neg\phi)\lra\Kw_i\neg(\chi\to\neg\phi)&\EquiKw\\
(iv)&\Kw_i(\chi\land\phi)\lra\Kw_i(\chi\to\neg\phi)&\RE,(ii),(iii)\\
(v)&(\neg\chi\land\phi)\lra\neg(\neg\chi\to\neg\phi)&\TAUT\\
(vi)&\Kw_i(\neg\chi\land\phi)\lra\Kw_i(\neg\chi\to\neg\phi)& (v),\text{ Similar to }(i)-(iv)\\
(vii)& \Kw_i(\chi\to\neg\phi)\land\Kw_i(\neg\chi\to\neg\phi)\to\Kw_i\neg\phi&\KwCon\\
(viii)& \Kw_i(\chi\land\phi)\land\Kw_i(\neg\chi\land\phi)\to\Kw_i\neg\phi&\RE,(vii),(iv),(vi)\\
(ix)&\Kw_i\phi\lra\Kw_i\neg\phi&\EquiKw\\
(x)&\Kw_i(\chi\land\phi)\land\Kw_i(\neg\chi\land\phi)\to\Kw_i\phi&\RE,(viii),(ix)\\
\end{array}
$$
\end{proof}

\begin{lemma}\label{usefulprop2}
$\vdash_\SPLKw\Kw_i\phi\to\Kw_i(\phi\land\psi)\vee\Kw_i(\neg\phi\land\chi)$
\end{lemma}

\begin{proof}
$$
\begin{array}{lll}
(i)&(\phi\land\psi)\lra\neg(\phi\to\neg\psi)&\TAUT\\
(ii)&\Kw_i(\phi\land\psi)\lra\Kw_i\neg(\phi\to\neg\psi)&\REKw,(i)\\
(iii)&\Kw_i(\phi\to\neg\psi)\lra\Kw_i\neg(\phi\to\neg\psi)&\EquiKw\\
(iv)&\Kw_i(\phi\land\psi)\lra\Kw_i(\phi\to\neg\psi)&\RE,(ii),(iii)\\
(v)&(\neg\phi\land\chi)\lra\neg(\neg\phi\to\neg\chi)&\TAUT\\
(vi)&\Kw_i(\neg\phi\land\chi)\lra\Kw_i(\neg\phi\to\neg\chi)&(v),\text{ Similar to }(i)-(iv)\\
(vii)&\Kw_i\phi\to\Kw_i(\phi\to\neg\psi)\vee\Kw_i(\neg\phi\to\neg\chi)&\KwDis\\
(viii)&\Kw_i\phi\to\Kw_i(\phi\land\psi)\vee\Kw_i(\neg\phi\land\chi)&\RE,(vii),(iv),(vi)\\
\end{array}
$$
\end{proof}

Using these lemmas, we now derive $\texttt{I2}$, $\texttt{I3}$, $\texttt{I4}$ and the rules $\texttt{RI}$, and $\RE$ in $\SPLKw$.

\begin{proposition} \label{prop.I2}
Axiom $\texttt{I2}$ is derivable in $\SPLKw$.
\end{proposition}

\begin{proof}
We derive the equivalent equivalent proposition:
$$\vdash_{\SPLKw}\Kw_i\phi\land\Kw_i\psi\to\Kw_i(\phi\land\psi)$$

$$
\begin{array}{lll}
(i)& \Kw_i\phi\to\Kw_i(\phi\land\psi)\vee\Kw_i(\neg\phi\land\neg(\phi\land\psi))&\text{Lemma~\ref{usefulprop2}}\\
(ii)&\Kw_i\psi\to\Kw_i(\psi\land\phi)\vee\Kw_i(\neg\psi\land\phi)&\text{Lemma~\ref{usefulprop2}}\\
(iii)&\Kw_i(\phi\land\psi)\lra\Kw_i(\psi\land\phi)&\TAUT,\REKw\\
(iv)&\Kw_i\psi\to\Kw_i(\phi\land\psi)\vee\Kw_i(\neg\psi\land\phi)&\RE,(ii),(iii)\\
(v)&(\neg\psi \land\phi)\lra\phi\land\neg(\phi\land\psi)&\TAUT\\
(vi)&\Kw_i(\neg\psi \land\phi)\lra\Kw_i(\phi\land\neg(\phi\land\psi))&\REKw,(v)\\
(vii)&\Kw_i\psi\to\Kw_i(\phi\land\psi)\vee\Kw_i(\phi\land\neg(\phi\land\psi))&\RE,(iv),(vi)\\
(viii)&\Kw_i(\phi\land\neg(\phi\land\psi))\land \Kw_i(\neg\phi\land\neg(\phi\land\psi))\to \Kw_i\neg(\phi\land\psi)&\text{Lemma~\ref{usefulprop1}}\\
(ix)&\Kw_i(\phi\land\psi)\lra \Kw_i\neg(\phi\land\psi)&\EquiKw\\
(x)&\Kw_i(\phi\land\neg(\phi\land\psi))\land \Kw_i(\neg\phi\land\neg(\phi\land\psi))\to \Kw_i(\phi\land\psi)&\RE,(viii),(ix)\\
(xi)& \Kw_i\phi\land\Kw_i\psi\to\Kw_i(\phi\land\psi)~\lor\\
&\ \ \ (\Kw_i(\phi\land\neg(\phi\land\psi))\land\Kw_i(\neg\phi\land\neg(\phi\land\psi)))&(i),(vii)\\
(xii)&\Kw_i\phi\land\Kw_i\psi\to\Kw_i(\phi\land\psi)&(x),(xi)
\end{array}
$$
\end{proof}

\begin{proposition}
Axiom $\texttt{I3}$ is derivable in $\SPLKw$.
\end{proposition}

\begin{proof}
First, we prove that $\vdash_{\SPLKw}\Kw_i\phi\land\neg\Kw_i(\chi_1\land\phi)
\land\Kw_i(\phi\to\psi)\to\Kw_i\psi$ \quad ($\star$).
\weg{$$
\begin{array}{lll}
(i)&\neg(\neg\phi\to\alpha_1\land\phi)\lra(\neg\phi\land\neg(\alpha_1\land\phi))&\TAUT\\
(ii)&\Kw_i\neg(\neg\phi\to\alpha_1\land\phi)\lra\Kw_i(\neg\phi\land\neg(\alpha_1\land\phi))&\REKw(i)\\
(iii)&\Kw_i(\neg\phi\to\alpha_1\land\phi)\lra\Kw_i\neg(\neg\phi\to\alpha_1\land\phi)&\EquiKw\\
(iv)&\Kw_i(\neg\phi\to\alpha_1\land\phi)\lra\Kw_i(\neg\phi\land\neg(\alpha_1\land\phi))&\RE(ii)(iii)\\
(v)&\Kw_i\phi\to\Kw_i(\phi\land\alpha_1)\vee\Kw_i(\neg\phi\land\neg(\alpha_1\land\phi))&\text{Lemma~\ref{usefulprop2}}\\
(vi)&\Kw_i\phi\land\neg\Kw_i(\alpha_1\land\phi)\to\Kw_i(\neg\phi\land\neg(\alpha_1\land\phi))&(v)\\
(vii)&\Kw_i\phi\land\Kw_i(\neg\phi\to(\alpha_1\land\phi))\land\neg\Kw_i(\alpha_1\land\phi)\land\Kw_i(\phi\to\psi)\to\Kw_i\psi&\text{Lemma}~ \ref{Mix}\\
(viii)&\Kw_i\phi\land\neg\Kw_i(\alpha_1\land\phi)\land\Kw_i(\phi\to\psi)\to\Kw_i\psi&\MP(vi)(vii)
\end{array}
$$}
$$
\begin{array}{lll}
(i)& \neg(\chi_1\land\phi)\lra(\phi\to\neg\chi_1)&\TAUT\\
(ii)&\Kw_i\neg(\chi_1\land\phi)\lra\Kw_i(\phi\to\neg\chi_1)&\REKw,(i)\\
(iii)&\Kw_i(\chi_1\land\phi)\lra\Kw_i\neg(\chi_1\land\phi)&\EquiKw\\
(iv)&\Kw_i(\chi_1\land\phi)\lra\Kw_i(\phi\to\neg\chi_1)&\RE,(ii),(iii)\\
(v)&\Kw_i\phi\to\Kw_i(\phi\to\neg\chi_1)\vee\Kw_i(\neg\phi\to(\chi_1\land\phi))&\KwDis\\
(vi)&\Kw_i\phi\land\neg\Kw_i(\phi\to\neg\chi_1)\to\Kw_i(\neg\phi\to(\chi_1\land\phi))&(v)\\
(vii)&\Kw_i\phi\land\neg\Kw_i(\chi_1\land\phi)\to\Kw_i(\neg\phi\to(\chi_1\land\phi))&\RE,(iv),(vi)\\
(viii)&\Kw_i\phi\land\Kw_i(\neg\phi\to(\chi_1\land\phi))\land\neg\Kw_i(\chi_1\land\phi)\land\Kw_i(\phi\to\psi)\\
&\ \ \ \to\Kw_i\psi&\text{Lemma}~\ref{Mix}\\
(ix)&\Kw_i\phi\land\neg\Kw_i(\chi_1\land\phi)\land\Kw_i(\phi\to\psi)\to\Kw_i\psi&(vii),(viii)
\end{array}
$$
Next, we prove that
$$\vdash_{\SPLKw}\Kw_i\phi\land\neg\Kw_i(\chi_1\land\phi)\land\Kw_i(\phi\to\psi)\land\neg\Kw_i(\chi_2\land(\phi\to\psi))\to\neg\Kw_i(\chi_1\land\psi).$$
$$
\begin{array}{lll}
(i)&\Kw_i\phi\land\Kw_i\psi\to\Kw_i(\phi\land\psi)&I2\\
(ii)&\Kw_i(\phi\land\psi)\land\Kw_i(\chi_1\land\psi)\to\Kw_i(\psi\land\phi\land\chi_1)&I2\\
(iii)&\Kw_i\phi\land\Kw_i\psi\land\Kw_i(\chi_1\land\psi)\to\Kw_i(\psi\land\phi\land\chi_1)&(i),(ii)\\
(iv)&\Kw_i(\psi\land\phi\land\chi_1)\land\Kw_i(\neg\psi\land\phi\land\chi_1)\to\Kw_i(\phi\land\chi_1)&\text{Lemma~\ref{usefulprop1}}\\
(v)&\Kw_i\phi\land\Kw_i\psi\land\Kw_i(\chi_1\land\psi)\land\Kw_i(\neg\psi\land\phi\land\chi_1)\to\Kw_i(\phi\land\chi_1)&(iii),(iv)\\
(vi)&\Kw_i(\phi\to\psi)\to\Kw_i((\phi\to\psi)\land\chi_2)\vee\Kw_i(\neg(\phi\to\psi)\land\chi_1)&\text{Lemma~\ref{usefulprop2}}\\
(vii)&\Kw_i(\phi\to\psi)\land\neg\Kw_i((\phi\to\psi)\land\chi_2)\to\Kw_i(\neg(\phi\to\psi)\land\chi_1)&(vi)\\
(viii)&\Kw_i(\neg\psi\land\phi\land\chi_1)\lra\Kw_i(\neg(\phi\to\psi)\land\chi_1)&\TAUT,\REKw\\
(ix)&\Kw_i(\phi\to\psi)\land\neg\Kw_i((\phi\to\psi)\land\chi_2)\to\Kw_i(\neg\psi\land\phi\land\chi_1)&(vii),(viii)\\
(x)&\Kw_i\phi\land\Kw_i\psi\land\Kw_i(\chi_1\land\psi)\land\Kw_i(\phi\to\psi)\land\neg\Kw_i((\phi\to\psi)\land\chi_2)\\
&\ \ \ \to\Kw_i(\phi\land\chi_1)&(v),(ix)\\
(xi)&\Kw_i\phi\land\Kw_i\psi\land\neg\Kw_i(\phi\land\chi_1)\land\Kw_i(\phi\to\psi)\land\neg\Kw_i((\phi\to\psi)\land\chi_2)\\
&\ \ \ \to\neg\Kw_i(\chi_1\land\psi)&(x)\\
(xii)&\Kw_i(\chi_1\land\phi)\lra\Kw_i(\phi\land\chi_1)&\TAUT,\REKw\\
(xiii)&\Kw_i(\chi_2\land(\phi\to\psi))\lra\Kw_i((\phi\to\psi)\land\chi_2)&\TAUT,\REKw\\
(xiv)&\Kw_i\phi\land\Kw_i\psi\land\neg\Kw_i(\chi_1\land\phi)\land\Kw_i(\phi\to\psi)\land\neg\Kw_i(\chi_2\land(\phi\to\psi))\\
&\ \ \ \to\neg\Kw_i(\chi_1\land\psi)&(xi),(xii),(xiii)\\
(xv)&
\Kw_i\phi\land\neg\Kw_i(\chi_1\land\phi)
\land\Kw_i(\phi\to\psi)\to\Kw_i\psi&(\star)\\
(xvi)&\Kw_i\phi\land\neg\Kw_i(\chi_1\land\phi)\land\Kw_i(\phi\to\psi)\land\neg\Kw_i(\chi_2\land(\phi\to\psi))\\
&\ \ \ \to\neg\Kw_i(\chi_1\land\psi)&(xiv),(xv)\\
\weg{(i)&\Kw_i\phi\to\Kw_i(\phi\to\neg\alpha_1)\vee\Kw_i(\neg\phi\to\delta)&\KwDis\\
(ii)&\Kw_i\phi\land\neg\Kw_i(\phi\land\alpha_1)\to\Kw_i(\neg\phi\to\delta)&(i)\\
(iii)&\Kw_i(\phi\to\psi)\to\Kw_i((\phi\to\psi)\to\neg\alpha_2)\vee\Kw_i(\neg(\phi\to\psi)\to\delta')&\KwDis\\
(iv)&\Kw_i(\phi\to\psi)\land\neg\Kw_i(\alpha_2\land(\phi\to\psi))\to\Kw_i(\neg(\phi\to\psi)\to\delta')&(iii)\\
(..)&\Kw_i\phi\land\Kw_i\psi\land\Kw_i(\phi\land\neg\psi\to\alpha_1)\land\Kw_i(\alpha_1\land\psi)\to\Kw_i(\alpha_1\land\phi)&..\\}
\end{array}
$$
Axiom $\texttt{I3}$ now follows from the two derived theorems by propositional reasoning.
\end{proof}

\begin{proposition}
Axiom $\texttt{I4}$ is derivable in $\SPLKw$.
\end{proposition}

\begin{proof}
A stronger result
$$\vdash_{\SPLKw}\neg\Kw_i\chi\to\neg\Kw_i(\chi\land\psi)\vee\neg\Kw_i(\chi\land\neg\psi)$$
follows directly from Prop.~\ref{usefulprop1}.
\end{proof}

\begin{proposition}
Inference rule $\texttt{RI}$ is admissible in $\SPLKw$.
\end{proposition}

\begin{proof}
Suppose that ${\vdash\phi}$. By $\GENKw$ we get ${\vdash\Kw_i\phi}$. We only need to show $\vdash\neg\Kw_i\chi\to\neg\Kw_i(\chi\land\phi)$, equivalently, $\vdash\Kw_i(\chi\land\phi)\to\Kw_i\chi$. From the supposition ${\vdash\phi}$, it follows $\vdash\neg\phi\to\neg\chi$. Then using $\GENKw$ again, we get $\vdash\Kw_i(\neg\phi\to\neg\chi)$. By Axiom $\KwCon$, we have $\vdash\Kw_i(\neg\phi\to\neg\chi)\land\Kw_i(\phi\to\neg\chi)\to\Kw_i\neg\chi$, thus by $\MP$ we get $\vdash\Kw_i(\phi\to\neg\chi)\to\Kw_i\neg\chi$. Note $\vdash\Kw_i(\chi\land\phi)\lra\Kw_i(\phi\to\neg\chi)$ and $\vdash\Kw_i\chi\lra\Kw_i\neg\chi$ can follow from $\EquiKw$. Therefore $\vdash\Kw_i(\chi\land\phi)\to\Kw_i\chi$, as desired.
\end{proof}

\begin{proposition}
\RE\ is admissible in $\SPLKw$.
\end{proposition}

\begin{proof}
Similar to the proof of Prop.~\ref{replacementofequivalent}.
\end{proof}

\section{Comparison with \cite{steinsvold:2008}} \label{sec.appendixb}

We prove that our $\SPLKwTTr$ and Steinsvold's ${\bf LB}$ are equivalent. We first show that the axioms and rules of the system ${\bf LB}$ are all derivable or admissible in $\SPLKwTTr$. We also use that $\Kw_i\phi\land\Kw_i\psi\to\Kw_i(\phi\land\psi)$ is derivable in $\SPLKw$ (see Prop.~\ref{prop.I2}).

\begin{proposition}
$\texttt{N}$, $\texttt{Z}$, $\texttt{R}$ are  derivable in $\SPLKwTTr$.
\end{proposition}
\begin{proof}
By $\TAUT$, $\GENKw$, $\EquiKw$ and Prop.~\ref{prop.I2}.
\end{proof}

\begin{proposition}
$\texttt{WM}$ is admissible in $\SPLKwTTr$.
\end{proposition}

\begin{proof}
Suppose that ${\vdash\Kw_i\phi\land\phi\to\psi}$. From the supposition and $\GENKw$, we have ${\vdash}\Kw_i(\Kw_i\phi\land\phi\to\psi)$. By $\vdash\Kw_i(\Kw_i\phi\land\phi\to\psi)\land\Kw_i(\Kw_i\phi\land\phi)\land(\Kw_i\phi\land\phi)\to\Kw_i\psi$ (Axiom $\KwT$),  we get $\vdash\Kw_i(\Kw_i\phi\land\phi)\land(\Kw_i\phi\land\phi)\to\Kw_i\psi$. Besides, from Axiom $\Tr$ we infer that $\vdash\Kw_i\phi\to\Kw_i\Kw_i\phi\land\Kw_i\phi$. With Prop.~\ref{prop.I2} we have $\vdash\Kw_i\Kw_i\phi\land\Kw_i\phi\to\Kw_i(\Kw_i\phi\land\phi)$. Now we conclude that $\vdash\Kw_i\phi\land\phi\to\Kw_i\psi$, as desired.
\end{proof}

Now we will prove that ${\bf LB}$ can derive $\SPLKwTTr$. Comparing the two systems, we only need to show that the axioms $\KwCon$, $\KwDis$, $\KwT$, and $\Tr$, and the rules $\GENKw$ and $\REKw$ can be derived in ${\bf LB}$. In the following derivation by $\vdash$ we mean $\vdash_{\bf LB}$.

\begin{lemma}\label{prop.lb}
If $\vdash\Kw_i\phi\land\phi\to\psi$, then $\vdash\Kw_i\phi\land\phi\to\Kw_i\psi$.
\end{lemma}
\begin{proof}
This is immediate from the rule $\texttt{WM}$.
\end{proof}

\begin{proposition}
$\KwCon$ is derivable in ${\bf LB}$.
\end{proposition}

\begin{proof}
$$
\begin{array}{lll}
(i)&((\chi\to\phi)\land(\neg\chi\to\phi))\lra\phi&\TAUT\\
(ii)&\Kw_i((\chi\to\phi)\land(\neg\chi\to\phi))\lra\Kw_i\phi&\RE,(i)\\
(iii)& \Kw_i(\chi\to\phi)\land\Kw_i(\neg\chi\to\phi)\to\Kw_i((\chi\to\phi)\land(\neg\chi\to\phi))&\texttt{R}\\
(iv)& \Kw_i(\chi\to\phi)\land\Kw_i(\neg\chi\to\phi)\to\Kw_i\phi& \RE,(ii),(iii) \\
\end{array}
$$
\end{proof}

\begin{proposition}
$\KwDis$ is derivable in ${\bf LB}$.
\end{proposition}

\begin{proof}
$$
\begin{array}{lll}
(i)&\phi\to(\neg\phi\to\chi)&\TAUT\\
(ii)& \Kw_i\phi\land\phi\to(\neg\phi\to\chi)&(i)\\
(iii)& \Kw_i\phi\land\phi\to\Kw_i(\neg\phi\to\chi)& \text{Lemma~\ref{prop.lb}},(ii)\\
(iv)& \neg\phi\to(\phi\to\psi)& \TAUT\\
(v)& \Kw_i\neg\phi\land\neg\phi\to\Kw_i(\phi\to\psi)&(iv),~\text{Similar to}~(i)-(iii)\\
(vi)& \Kw_i\phi\land\neg\phi\to\Kw_i(\phi\to\psi)&(v),\texttt{Z}\\
(vii)& \Kw_i\phi\to\Kw_i(\phi\to\psi)\vee\Kw_i(\neg\phi\to\chi)& (iii),(vi)\\
\end{array}
$$
\end{proof}

\weg{\begin{proof}
$$
\begin{array}{lll}
(i)&\phi\to(\neg\phi\to\chi)&\TAUT\\
(ii)& \Kw_i\phi\land\phi\to(\neg\phi\to\chi)&(i)\\
(iii)& \Kw_i\phi\land\phi\to\Kw_i(\neg\phi\to\chi)\land(\neg\phi\to\chi)&WM,(ii)\\
(iv)& \Kw_i\phi\land\phi\to\Kw_i(\neg\phi\to\chi)& (iii)\\
(v)& \neg\phi\to(\phi\to\psi)& \TAUT\\
(vi)& \Kw_i\neg\phi\land\neg\phi\to\Kw_i(\phi\to\psi)&(v),~\text{Similar to}~(i)-(iv)\\
(vii)& \Kw_i\phi\land\neg\phi\to\Kw_i(\phi\to\psi)&(vi),\texttt{Z}\\
(viii)& \Kw_i\phi\to\Kw_i(\phi\to\psi)\vee\Kw_i(\neg\phi\to\chi)& (iv),(vii)\\
\end{array}
$$
\end{proof}}

\weg{\begin{proposition}
$\EquiKw$ is derivable in ${\bf LB}$.
\end{proposition}

\begin{proof}
Immediate from \RE.
\end{proof}}

\begin{proposition}
$\KwT$ is derivable in ${\bf LB}$.
\end{proposition}

\begin{proof}
$$
\begin{array}{lll}
(i)& \phi\land(\phi\to\psi)\to\psi& \TAUT\\
(ii)& \Kw_i(\phi\land(\phi\to\psi))\land(\phi\land(\phi\to\psi))\to\psi& (i)\\
(iii)& \Kw_i(\phi\land(\phi\to\psi))\land(\phi\land(\phi\to\psi))\to\Kw_i\psi& \text{Lemma~\ref{prop.lb}},(ii)\\
(iv)& \Kw_i\phi\land\Kw_i(\phi\to\psi)\to\Kw_i(\phi\land(\phi\to\psi))&\texttt{R}\\
(v)& \Kw_i\phi\land\Kw_i(\phi\to\psi)\land\phi\land(\phi\to\psi)\to\Kw_i\psi&(iii),(iv)\\
(vi)& \neg(\phi\to\psi)\to\neg\psi& \TAUT\\
(vii)& \Kw_i\neg(\phi\to\psi)\land\neg(\phi\to\psi)\to\Kw_i\neg\psi& (vi),\text{ Similar to } (i)-(iii) \\
(viii)& \Kw_i(\phi\to\psi)\land\neg(\phi\to\psi)\to\Kw_i\psi& (vii),\texttt{Z}\\
(ix)& \Kw_i\phi\land\Kw_i(\phi\to\psi)\land\phi\land\neg(\phi\to\psi)\to\Kw_i\psi&(viii)\\
(x)& \Kw_i\phi\land\Kw_i(\phi\to\psi)\land\phi\to\Kw_i\psi& (v),(ix)\\
\end{array}
$$
\end{proof}

\begin{proposition}
$\Tr$ is derivable in ${\bf LB}$.
\end{proposition}

\begin{proof}
$$
\begin{array}{lll}
(i)& \Kw_i\phi\land\phi\to\Kw_i\phi& \TAUT\\
(ii)& \Kw_i\phi\land\phi\to\Kw_i\Kw_i\phi& \text{Lemma~\ref{prop.lb}},(i)\\
(iii)& \Kw_i\neg\phi\land\neg\phi\to\Kw_i\neg\phi& \TAUT\\
(iv)& \Kw_i\neg\phi\land\neg\phi\to\Kw_i\phi&(iii),\texttt{Z}\\
(v)& \Kw_i\neg\phi\land\neg\phi\to\Kw_i\Kw_i\phi&\text{Lemma~\ref{prop.lb}},(iv)\\
(vi)&\Kw_i\phi\land\neg\phi\to\Kw_i\Kw_i\phi&(v),\texttt{Z}\\
(vii)&\Kw_i\phi\to\Kw_i\Kw_i\phi& (ii),(vi)\\
\end{array}
$$
\end{proof}

\begin{proposition}
$\GENKw$ is admissible in ${\bf LB}$.
\end{proposition}

\begin{proof}
Suppose that $\vdash\phi$, we need to show $\vdash\Kw_i\phi$. From the supposition follows that $\vdash\phi\lra\top$, then by \RE\ we get $\vdash\Kw_i\phi\lra\Kw_i\top$. By \texttt{N} and $\TAUT$, we have $\vdash\Kw_i\top$, and hence we conclude that $\vdash\Kw_i\phi$.
\end{proof}

\begin{proposition}
$\REKw$ is admissible in ${\bf LB}$.
\end{proposition}

\begin{proof}
Immediate from \RE.
\end{proof}

\end{document}